\documentclass[sigconf]{acmart}

\usepackage{amsmath}
\usepackage{graphicx}
\usepackage{multirow}
\usepackage{amsthm}
\usepackage{wrapfig}
\usepackage{float} 
\usepackage{subfigure}  
\usepackage{algorithm}
\usepackage{algorithmic}
\usepackage{enumitem}
\usepackage{xcolor}
\usepackage{balance}

\newtheorem{lemma}{Lemma}

\AtBeginDocument{%
  }

\copyrightyear{2025}
\acmYear{2025}
\setcopyright{acmlicensed}\acmConference[WSDM '25]{Proceedings of the
Eighteenth ACM International Conference on Web Search and Data
Mining}{March 10--14, 2025}{Hannover, Germany}
\acmBooktitle{Proceedings of the Eighteenth ACM International Conference on
Web Search and Data Mining (WSDM '25), March 10--14, 2025, Hannover,
Germany}
\acmDOI{10.1145/3701551.3703489}
\acmISBN{979-8-4007-1329-3/25/03}

\newcommand{\model}{{CorDGT}}
\newcommand{\modell}{{CorDGT\ }}

\begin{document}

\title{Dynamic Graph Transformer with Correlated Spatial-Temporal Positional Encoding}

\author{Zhe Wang}
\affiliation{%
    \institution{Zhejiang University}
    \department{College of Computer Science}
    \city{Hangzhou}
    \country{China}
}
\email{zhewangcs@zju.edu.cn}

\author{Sheng Zhou}
\authornote{Corresponding author.}
\affiliation{%
    \institution{Zhejiang University}
    \department{School of Software Technology}
    \city{Ningbo}
    \country{China}
}
\email{zhousheng\_zju@zju.edu.cn}

\author{Jiawei Chen}
\affiliation{
    \institution{Zhejiang University}
    \department{The State Key Laboratory of Blockchain and Data Security}
    \city{Hangzhou}
    \country{China}
}
\email{sleepyhunt@zju.edu.cn}

\author{Zhen Zhang}
\affiliation{%
    \institution{National University of Singapore}
    \city{Singapore}
    \country{Singapore}
}
\email{zhen@nus.edu.sg}

\author{Binbin Hu}
\affiliation{%
    \institution{Ant Group}
    \city{Hangzhou}
    \country{China}
}
\email{bin.hbb@antfin.com}

\author{Yan Feng}
\affiliation{%
    \institution{Zhejiang University}
    \department{The State Key Laboratory of Blockchain and Data Security}
    \city{Hangzhou}
    \country{China}
}
\email{fengyan@zju.edu.cn}

\author{Chun Chen}
\affiliation{%
    \institution{Zhejiang University}
    \department{The State Key Laboratory of Blockchain and Data Security}
    \city{Hangzhou}
    \country{China}
}
\email{chenc@zju.edu.cn}

\author{Can Wang}
\affiliation{%
    \institution{Zhejiang University}
    \department{The State Key Laboratory of Blockchain and Data Security}
    \city{Hangzhou}
    \country{China}
}
\email{wcan@zju.edu.cn}

\renewcommand{\shortauthors}{Zhe Wang et al.}

\begin{abstract}
Learning effective representations for Continuous-Time Dynamic Graphs (CTDGs) has garnered significant research interest, largely due to its powerful capabilities in modeling complex interactions between nodes.
A fundamental and crucial requirement for representation learning in CTDGs is the appropriate estimation and preservation of proximity. 
However, due to the sparse and evolving characteristics of CTDGs, the spatial-temporal properties inherent in high-order proximity remain largely unexplored.
Despite its importance, this property presents significant challenges due to the computationally intensive nature of personalized interaction intensity estimation and the dynamic attributes of CTDGs.
To this end, we propose a novel Correlated Spatial-Temporal Positional encoding that incorporates a parameter-free personalized interaction intensity estimation under the weak assumption of the Poisson Point Process.
Building on this, we introduce the \textbf{D}ynamic \textbf{G}raph \textbf{T}ransformer with \textbf{Cor}related Spatial-Temporal Positional Encoding (\textbf{CorDGT}), which efficiently retains the evolving spatial-temporal high-order proximity for effective node representation learning in CTDGs. Extensive experiments on seven small and two large-scale datasets demonstrate the superior performance and scalability of the proposed \model. The code is available at: \url{https://github.com/wangz3066/CorDGT}.

\end{abstract}

\ccsdesc[500]{Theory of computation~Dynamic graph algorithms}

\keywords{Continuous-Time Dynamic Graph, Graph Representation Learning, Transformer, Proximity}



\maketitle

\section{Introduction}
\label{sec::intro}

Graph Neural Networks (GNNs) \cite{velickovic2018graph, hamilton2017inductive, xu2018how, dong2021equivalence, wang2024online} have become a potent tool for analyzing diverse graph structures due to their effectiveness in learning low-dimensional graph representations. While early GNN research focused on static graphs, many real-world network data exhibit evolving graph structures, such as the World Wide Web and recommendation systems \cite{chen2023bias, wang2024distributionally, wu2022graph}. 
The increasing prevalence of dynamic graph data has spurred researchers to adapt GNNs to handle dynamic graphs.
Dynamic graphs can be categorized into Discrete-Time Dynamic Graphs (DTDGs) and Continuous-Time Dynamic Graphs (CTDGs). 
Recently, CTDGs have garnered more research attention compared to DTDGs, attributed to their flexible generalization capabilities and proficiency in modeling intricate and dynamic node interactions.
To learn effective representations for CTDGs, researchers have explored various techniques, including temporal message passing \cite{xuiclr2020tgat,rossi2020temporal, kumar2019predicting, trivedi2019dyrep, wang2024towards} and the incorporation of temporal intervals into random walks \cite{nguyen2018continuous, wangiclr2021caw, jin2022neural}, etc. 

The effectiveness of graph representation fundamentally hinges on the preservation of proximity between nodes. 
This implies that nodes that are proximate in the graph should also maintain closeness in the low-dimensional space.
The first-order proximity, represented by the direct interactions between node pairs observed in the CTDG, can be easily maintained by making the embeddings of adjacent nodes close. However, considering the sparsity and the evolving nature of CTDGs, the preservation of first-order proximity alone does not provide a comprehensive measure of proximity.
The high-order proximity \cite{tang2015line} in CTDGs, evaluated by the closeness degree between target nodes and auxiliary nodes, embodies the \textbf{spatial-temporal duality}. 
Figure \ref{fig::proximity} illustrates this spatial-temporal duality of high-order proximity in CTDGs using a social network example. Suppose the model is predicting the interaction between users $u$ and $v$ at time $t=11$. As there is not direct interaction between the target nodes, their shared neighbors should be taken into consideration. Unlike the static graphs, although users $u$ and $v$ have three shared neighbors in both subgraphs (a) and (b), they have more frequent and recent interactions with their shared neighbors in subgraph (a) than (b), indicating higher probability of connections between node pairs $u$ and $v$. 
Although important, most existing methods \cite{trivedi2019dyrep, xuiclr2020tgat, rossi2020temporal} independently aggregate neighbor information of target nodes without considering dependencies between target nodes.
Some existing works \cite{wangiclr2021caw, jin2022neural} encode the appearance positions of auxiliary nodes on the random walks path, which fail to address the spatial-temporal duality property inherent in high-order proximity.

Despite its paramount importance, estimating and preserving comprehensive spatial-temporal proximity in CTDGs poses several challenges for the following reasons: \textit{Firstly}, CTDGs entail multiple interactions occurring between node pairs at varying timestamps. The \textbf{interaction intensity}, which represents the count of interactions within a specified time interval, is instrumental in characterizing the degree of connection between node pairs. Consequently, measuring higher-order proximity in CTDGs necessitates a personalized intensity estimation between the target node pair and the auxiliary nodes at any given timestamp, a process that is computationally demanding. Several existing works \cite{zuo2018embedding, wen2022trend} employ the Hawkes process and learn node representations to estimate pairwise intensity. Nonetheless, integrating these methods to estimate high-order proximity necessitates pre-training of the node embeddings or neural networks, thereby incurring significant computational costs. \textit{Secondly}, the dynamic nature of CTDGs leads to varying proximity between node pairs across different timestamps, influenced by both spatial and temporal factors. This necessitates efficient adaptation in proximity estimation, thereby introducing unique challenges in the modeling of CTDGs.

To address the aforementioned challenges, in this paper, we propose \textbf{\model}, a \textbf{D}ynamic \textbf{G}raph \textbf{T}ransformer with \textbf{Cor}related Spatial-Temporal Positional Encoding. Inspired by the effectiveness of Poisson process in modeling counting process, we propose a novel Temporal Distance that incorporates a parameter-free interaction intensity estimation by leveraging the weak assumption of the Poisson Point Process. This approach circumvents the computational cost of pretraining, thereby enhancing efficiency. Based on Temporal Distance, we further propose a \textit{Correlated Spatial-Temporal Positional Encoding (STPE-C)} that models the spatial-temporal duality of evolving high-order proximity in CTDGs. Equipped with the STPE-C, we propose a dynamic graph Transformer that adaptively preserves the comprehensive proximity for effective learning of node embeddings in CTDGs, which are subsequently leveraged for downstream tasks. Consistent performance improvement over eight baselines are observed on both link prediction and node classification tasks, demonstrating the superiority of the proposed \model. Additionally, experiments on large-scale datasets demonstrate the superior effectiveness-efficiency trade-off of proposed \model. In summary, the main contributions of this paper are as follows:
\begin{itemize}
\setlength{\leftmargin}{0pt}
\item We propose a novel estimation of comprehensive proximity that incorporates an efficient parameter-free personalized intensity to encode the evolving spatial-temporal high-order proximity on CTDGs.
\item We propose Dynamic Graph Transformer with Correlated Spatial-Temporal Positional Encoding (\model), which efficiently preserves the evolving comprehensive proximity for effective node representation learning in CTDGs.
\item Extensive experiments conducted on seven small and two large-scale datasets demonstrate the superiority and the scalability of the proposed \model. 
\end{itemize}

\begin{figure}[t]
    \centering
    \includegraphics[width=0.9\linewidth]{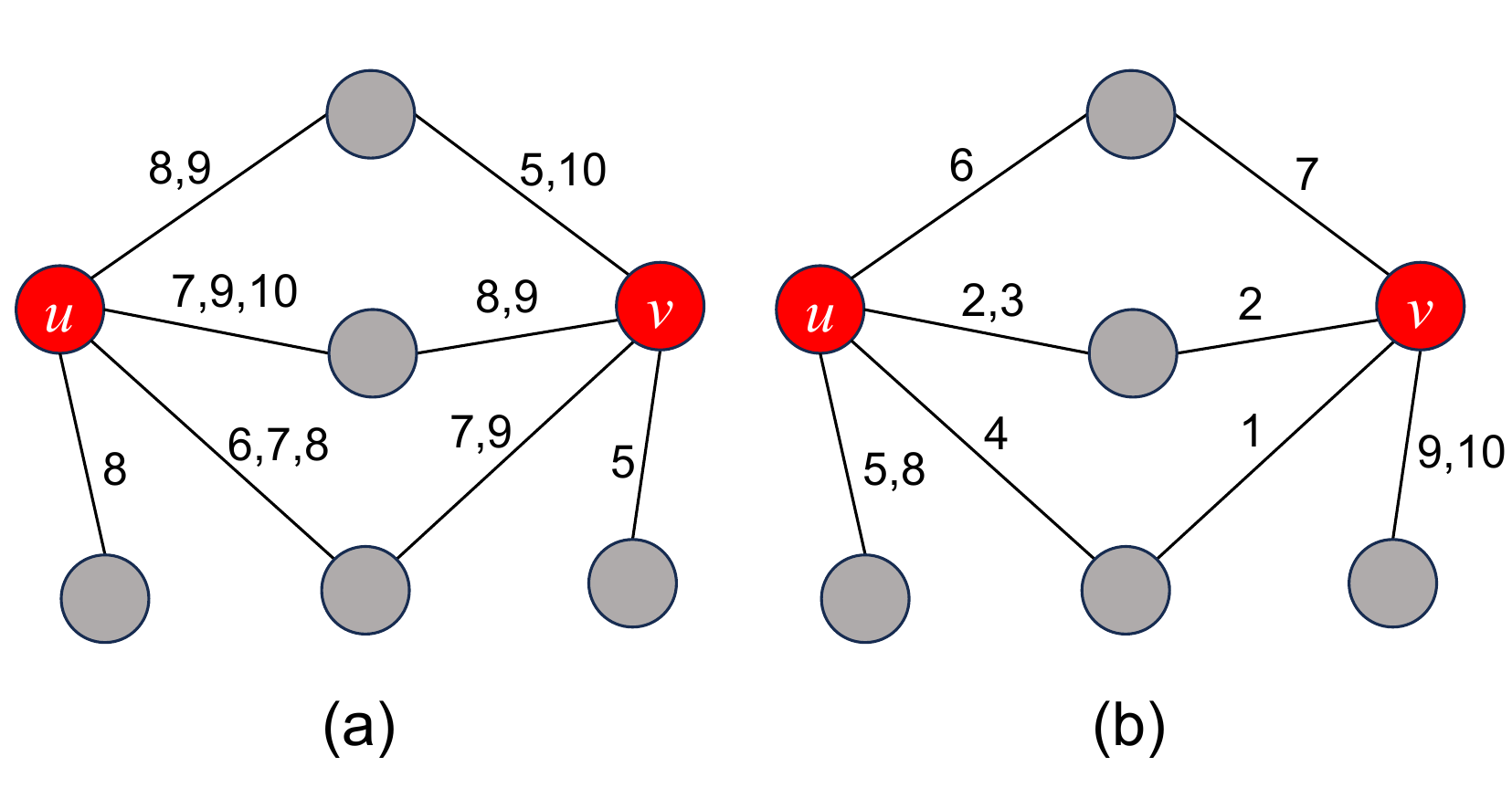}
    \caption{A social network example. The model is expected to predict the existence of the interaction between node $u$ and $v$ at $t=11$. }
    \label{fig::proximity}
\end{figure}

\section{Related Works}
\paragraph{\textbf{Dynamic Graph Neural Networks.}}
Existing dynamic graphs can be categorized into Discrete-Time Dynamic Graphs (DTDGs) and Continuous-Time Dynamic Graphs (CTDGs) based on whether the timestamps in the dynamic graphs are discrete or continuous. Early works on DTDGs learn the graph representation of each snapshot, which is then fed into a sequential model such as a Recurrent Neural Network or Transformer to learn temporal representation \cite{yu2018spatio,li2018diffusion, pareja2020evolvegcn,sankar2020dysat, you2022roland}. Some recent works learn the representation of successive snapshots selected by Bernoulli sampling \cite{zhang2023dyted} or a sliding window \cite{zhu2023wingnn}. In contrast, the graph proximity of CTDGs is highly coupled with the time series. Most CTDG models leverage a unified model to learn the node representation via temporal graph neural networks \cite{xuiclr2020tgat, rossi2020temporal, trivedi2019dyrep, kumar2019predicting, cong2023we, luo2022neighborhood}, random walks \cite{nguyen2018continuous, wangiclr2021caw, jin2022neural}, or Transformers \cite{wang2021apan, fan2021continuous, yu2023towards, cong2021dynamic}. However, due to the sparsity and evolving nature of CTDGs, addressing the spatial-temporal property of high-order proximity is important for proximity measurement, which is overlooked by existing works. Different from existing works, the proposed model addresses the spatial-temporal high-order proximity in CTDGs by incorporating the spatial-temporal distance between the target and the auxiliary nodes.

\paragraph{\textbf{Temporal Point Process.}} The Temporal Point Process (TPP) is a mathematical model used to represent a sequence of events in time, which has been employed to model the interaction intensity of Continuous-Time Dynamic Graphs (CTDGs) \cite{trivedi2019dyrep, chang2020continuous, wen2022trend}. These works have utilized a parametric network or temporal node embeddings to model the interaction intensity of TPPs, such as the Hawkes process \cite{hawkes1971spectra}. However, the pre-training process of these methods can be time-consuming, making them unsuitable for high-order proximity modeling. Instead, we propose a parameter-free intensity estimation method based on the Poisson Point Process.

\paragraph{\textbf{Graph Transformers.}}
Several studies have explored the application of the pure Transformer model in static graph representation learning, wherein graph-specific information is incorporated as a soft inductive bias via positional encodings such as eigenvectors of the graph Laplacian matrix \cite{dwivedi2020generalization, kreuzer2021rethinking}, diagonals of the random walk matrix \cite{dwivedi2021graph}. pairwise shortest path length \cite{ying2021graphormer, park2022grpe}. In the context of DTDGs, several works have proposed utilizing the Transformer model to capture the temporal evolution following spatial graph convolution within each graph snapshot \cite{xu2020spatial, yu2020spatio, sankar2020dysat, wang2022synchronous}. Given the Transformer model's capability to learn long-term dependencies, it has been adopted for learning on CTDGs \cite{fan2021continuous, yu2023towards}. APAN \cite{wang2021apan} employs the Transformer to model asynchronous mail messages from other temporal neighbors. 
Our proposed model differs from these methods in terms of input tokens, positional encodings, and Transformer architecture.

\section{Problem Definition}
\paragraph{\textbf{Problem Formulation.}}
The Continuous-Time Dynamic Graph is defined as a set of interaction events $\mathcal{E} = [(u_i,v_i,t_i,e_i)]_{i=1}^M$, where $M$ is the number of events, and the tuple $(u_i, v_i, t_i)$ represents that nodes $u_i$ and $v_i$ interact at time $t_i$, and $e_i \in \mathbb{R}^{d_e}$ is the feature vector associated with the $i$-th interaction event. The node feature matrix of the CTDG is denoted as $\mathcal{X}=[x(1),...,x(N)] \in \mathbb{R}^{N \times d_n}$ where $N$ is the total number of nodes and $x(i)$ is the raw node feature of the $i$-th node. Two nodes may have multiple interactions at different timestamps in a CTDG. The objective of representation learning on CTDG is to learn a embedding function for each node $i$ at time $t$: $f_i(t) \in \mathbb{R}^d$. Since most of the CTDG datasets do not have node labels, the CTDG models are typically trained based on the future link prediction task. Future link prediction on CTDGs aims to predict the occurrence probability of the link $(u,v,t_{pred})$ based on all the historical interactions happening before $t_{pred}$, which can be categorized into transductive and inductive settings based on whether the testing nodes are visible in the training stage. 
The network parameters trained based on transductive link prediction can be utilized for other downstream tasks such as node classification.

\paragraph{\textbf{Temporal Neighbor \cite{xuiclr2020tgat, rossi2020temporal}.}} Given a node $u$ at time $t$, the collection of its (1-hop) temporal neighbors is defined as the set of nodes that have interaction with $u$ before $t$: $\mathcal{\eta}^{(1)}(u;t)=\{(w,t')| \\ (u,w, t',\cdot) \in \mathcal{E}, t'<t \}$. The collection of $K$-hop ($K>1$) temporal neighbors of the node $u$ at time $t$ can be recursively defined as all the temporal neighbor of its $(K-1)$-hop temporal neighbors, denoted as $\mathcal{\eta}^{(K)}(u;t)$. The $K$-hop temporal neighborhood of a node $u$ at time $t$ is defined as $\mathcal{N}^{(K)}(u;t)= \cup_{i=1}^K \mathcal{\eta}^{(i)}(u;t)$.

\section{Model}
\label{section::model}

\subsection{Overall Framework}
In this section, we introduce the proposed \textbf{D}ynamic \textbf{G}raph \textbf{T}rans-former with \textbf{Cor}related Spatial-Temporal Positional Encoding \\ (\textbf{CorDGT}).
The general framework of \modell is presented in Figure \ref{fig::framework}. 
Suppose we are predicting the interaction probability of the target nodes $u$ and $v$ at time $t_{pred}$, \modell begins by sampling their \textit{contextual nodes set} from their $K$-hop temporal neighborhood, denoted as $\mathcal{C}(u,t_{pred})$ and $\mathcal{C}(v,t_{pred})$, respectively. As stated in Section \ref{sec::intro}, the spatial-temporal high-order proximity between the target nodes $u$ and $v$ is characterized by their first-order proximity to these contextual nodes $w \in \mathcal{C}(u,t_{pred}) \cup \mathcal{C}(v,t_{pred})$. Therefore, we first introduce Spatial Distance and Temporal Distance to characterize the first-order proximity on the CTDG. Then, the Correlated Spatial-Temporal Positional Encodings (STPE-C) for each contextual node $w \in \mathcal{C}(u,t_{pred}) \cup \mathcal{C}(v,t_{pred})$ with respect to the target node pairs $u$ and $v$ at time $t_{pred}$ is proposed to encode the spatial-temporal high-order proximity. Further, the network architecture of \modell is presented, which modifies the original Transformer to incorporate structural information of the CTDG.

\begin{figure*}
    \centering
    \includegraphics[width=0.88\linewidth]{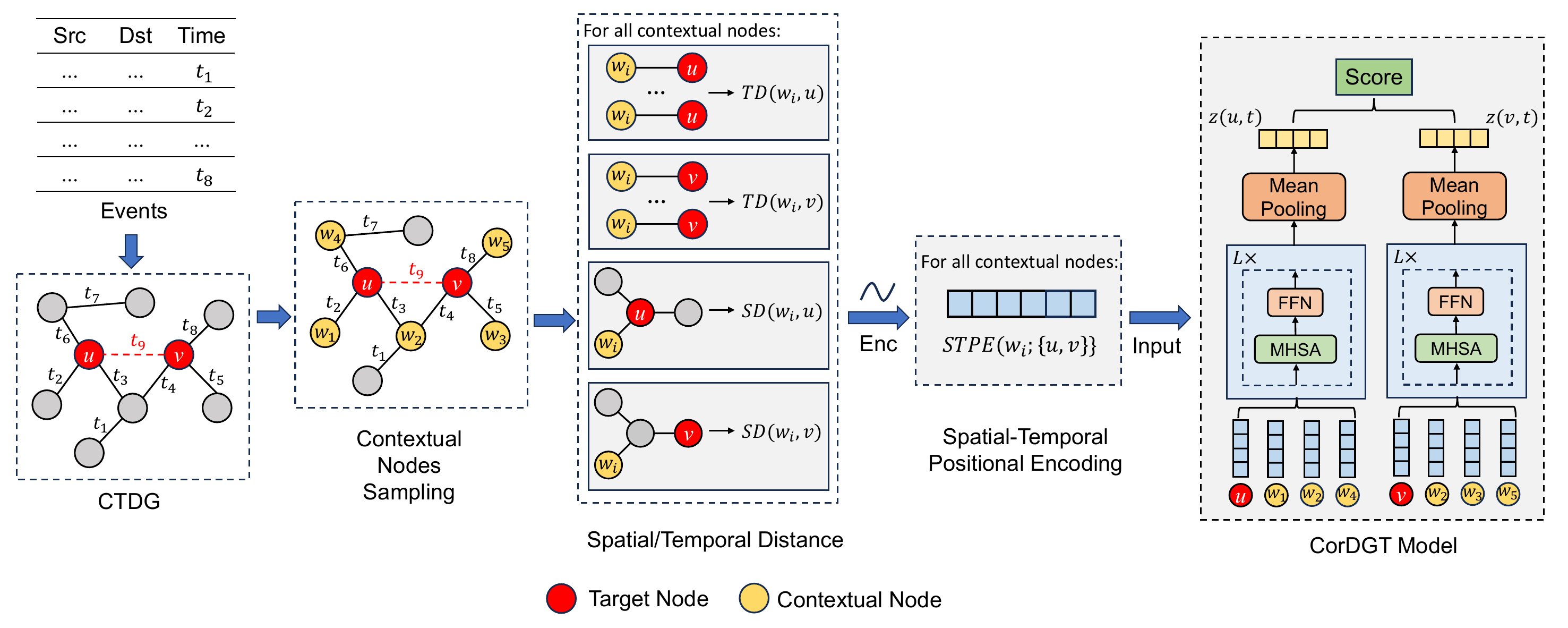}
    \caption{The framework of the proposed \model.}
    \label{fig::framework}
\end{figure*}

\subsection{Contextual Nodes Sampling}
\label{subsection::contexutal_nodes_sampling}
A simple tree-based sampling strategy is adopted to obtain contextual nodes set. Specifically, given the node $u$ at time $t_{pred}$, its contextual nodes set is initialized as $\mathcal{C}(u,t_{pred})=\{(u,t_{pred})\}$. At the first iteration, we uniformly sample $n_1$ temporal neighbors of $(u,t_{pred})$ and add them to $\mathcal{C}(u,t_{pred})$. At the $k$-th iteration ($k>1$), we uniformly sample $n_k$ temporal neighbors for each nodes sampled at the $(k-1)$ iteration and add them to $\mathcal{C}(u,t_{pred})$. Note that the sampling numbers $n_i$ ($i=1,...,K)$ are predefined hyper-parameters.

\subsection{Correlated Spatial-Temporal Positional Encoding}
\label{section::stpe-c}

\subsubsection{{Spatial Distance.}}
\label{subsection::spatial_distance}
The spatial distance can be characterized by the shortest path length on the topological structure of the CTDG. However, since the topological structure of the CTDG is continuously evolving as new interaction events happen, online updating the shortest path length between any two nodes on the entire CTDG is time consuming. Alternatively, we use the shortest path length on the sampled K-hop temporal neighborhood as a proxy. Specifically, to compute Spatial Distance (SD) of $w$ to $w_0$ at time $t_{pred}$, denoted as $\mathrm{SD}(w;w_0,t_{pred})$, we firstly sample K-hop temporal neighborhood from the root node $w_0$, then $\mathrm{SD}(w;w_0,t_{pred})$ is defined as:
\begin{equation}
\begin{split}
    \mathrm{SD}(w;w_0, t_{pred}) = \min \{ hop(w'; \mathcal{C}(w_0,t_{pred})|  \\ w=w', w' \in \mathcal{C}(w_0,t_{pred}) \}
\end{split}
\end{equation}
where $hop$ denotes the hop number of $w'$ from the root node $w_0$. Note that if If $w=w_0$, $ \mathrm{SD}(w;w_0,t_{pred})$ is set as 0. The computation of SD is based on the closest $w$ to the root node $w_0$, since $w$ may occur multiple times in the contextual node set of $w_0$. If $w$ is not in $\mathcal{C}(w_0, t_{pred})$, we set $\mathrm{SD}(w;w_0, t_{pred})$ as infinity. 

\subsubsection{Temporal Distance}
\label{subsection::temporal_distance}
In contrast to static graphs, the node pairs may have multiple interactions at different timestamps in the CTDG, thus the proximity between two nodes is associated with their interaction history. We propose Temporal Distance to characterize the proximity originated from the interaction history between two nodes. Suppose current timestamp is $t_{pred}$, and the timestamps sequence that nodes $w$ and $w_0$ interacted prior to $t_{pred}$ is denoted as $T(w,w_0,t_{pred}) = \{t_1, ..., t_n\}$ with $t_{i-1} < t_{i} (i=2,...,n)$ and $t_n < t_{pred}$. Then, the Temporal Distance (TD) between $w$ and $w_0$ at time $t_{pred}$, denoted as $\mathrm{TD}(w,w_0,t_{pred})$ is defined as:
\begin{equation}
\label{equation::TD1}
    \mathrm{TD}(w,w_0,t_{pred}) = f(\{t_1, ..., t_n\}, t_{pred})
\end{equation}
where $f$ is an arbitrary function which maps the interaction timestamps sequence to a scalar. Moreover, the temporal distance should satisfy following properties: (Recentness) If the most recent interaction between two nodes is closer to current time, then the temporal distance between them should be smaller; (Intensity) If the interaction intensity of these two nodes is higher, then the temporal distance between them should be smaller. The recentness property can be easily characterized by the difference between $t_n$ and $t_{pred}$. 
However, estimating the interaction intensity between two nodes is not straightforward since there is no prior knowledge to its distribution and the interactions patterns of different node pairs may be highly divergent. 
Most existing works adopt Hawkes process in modeling the interaction intensity \cite{zuo2018embedding, wen2022trend}. 
However, these methods require to pre-train the node embeddings to obtain the intensity, thus is time-consuming. 
Instead, we propose a parameter-free approach to estimate the interaction intensity at any time. 
Specifically, we employ the Poisson point process assumption for its simplicity and generalization ability, which is a commonly used weak assumption for an unknown counting process \cite{lavenberg1983computer, wangiclr2021caw}.
We provide following Theorem \ref{theorem::poisson} with the Poisson point process assumption to evaluate maximum likelihood estimation of interaction intensity given the interaction sequence.
\begin{lemma}
\label{theorem::poisson}
Suppose the interactions between $w$ and $w_0$ prior to $t_{pred}$, denoted as $T(w,w_0,t_{pred})= \{t_1, ..., t_n \}$ with $t_{i-1} < t_{i}(i=2,...,n)$ and $t_{n} < t_{pred}$, follow a Poisson point process with intensity $\lambda$, then the maximum likelihood estimation of $\lambda$ is $\dfrac{n}{t_n}$. 
\end{lemma}

\begin{proof}
With the Poisson Point Process assumption, the probability of the interaction sequence $[t_1,...,t_n]$ happens is:
\begin{equation}
    Pr(N(t_{i-1},t_i)=1, i=1,2,...,n) = \prod_{i=1}^n \dfrac{{\lambda(t_i-t_{i-1})}^1}{1!}e^{-\lambda(t_i-t_{i-1})}
\end{equation}
where $N(t_{i-1},t_i)$ denotes the number of interactions within the range $(t_{i-1},t_i)$ and $t_0=0$. Therefore, the likelihood function of the intensity $\lambda$ can be written as: 
\begin{equation}
    L(\lambda) = \prod_{i=1}^n \lambda (t_i - t_{i-1}) \cdot e^{-\lambda (t_i - t_{i-1})}
\end{equation}
Thus, the log-likelihood of $\lambda$ is: 
\begin{equation}
    l(\lambda) = \sum_{i=1}^n \ln(\lambda) + \ln(t_i - t_{i-1}) - \lambda (t_i - t_{i-1})
\end{equation}
By setting the derivative $\dfrac{dl}{d\lambda}=0$, we get the maximum likelihood estimation of $\lambda$: 
\begin{equation}
    \lambda_{MLE} = \dfrac{n}{t_n}
\end{equation}
which concludes the proof. 
\end{proof}

Lemma \ref{theorem::poisson} enables us to characterize the interaction intensity by $n / t_n$. By integrating both the intensity property and recentness property, we specify the temporal distance defined in Eq. (\ref{equation::TD1}) as: 
\begin{equation}
\label{equation::TD2}
    \mathrm{TD}(w,w_0,t_{pred}) = \alpha * \dfrac{t_n}{t_{pred} \cdot n} + \beta * \dfrac{t_{pred}-t_n}{t_{pred}} 
\end{equation}
where $\alpha, \beta>0$ are hyper-parameters. Note that if $w$ and $w_0$ do not have interaction before, we set $\mathrm{TD}(w,w_0, t_{pred})$ as a very large value. If $w=w_0$, $ \mathrm{TD}(w,w_0,t_{pred})$ is set as 0. For implementation, due to the sparsity of CTDGs, we only need to store the most recent interaction timestamp $t_n$ and the interaction times count $n$ for the node pairs that have interactions before, which leads to the memory complexity significantly less than $\mathcal{O}(|\mathcal{E}|)$. To prevent the problem of information leakage, when computing the temporal distance in Eq. (\ref{equation::TD2}), we use the recorded $t_n$ and $n$  until the \textit{previous} mini-batch of interactions. After the complete forward propagation of the current mini-batch, we update $t_n$ and $n$ records using the interactions of the current mini-batch.


\subsubsection{{Spatial-Temporal Positional Encoding.}}
\label{subsection::stpe-c}
The aforementioned spatial distance and temporal distance encode the direct proximity between two nodes on CTDG in scalars. However, the expressiveness of self-attention will be restricted if we directly use the scalar distance as inputs. In addition, this encoding function should learn the difference of spatial-temporal distance among contextual nodes more effectively. Inspired by \cite{vaswani2017attention}, we use the sinusoidal function $\mathrm{Enc}: \mathbb{R^+} \rightarrow \mathbb{R}^{2d}$ as the encoding function: 
\begin{equation}
\label{eq::STPE_u}
\begin{split}
\mathrm{Enc}(x)[2i] = \sin\left(\frac{\epsilon x}{10000^{2i/d}}\right) \\
\mathrm{Enc}(x)[2i+1] = \cos\left(\frac{\epsilon x}{10000^{2i/d}}\right)
\end{split}
\end{equation}
where $\epsilon$ is used to amplify the influence of $x$ on different positions of the encoding, and we set $\epsilon=10000$ in this work. Further, the Unitary Spatial-Temporal Positional Encoding (STPE-U) of the contextual node $w$ with respect to the single target node $w_0$ at time $t_{pred}$ is defined as: 
\begin{equation}
\begin{split}
    \mathrm{STPE\_u}(w;w_0, t_{pred}) = &\mathrm{MLP}(\mathrm{Enc}(\mathrm{TD}(w,w_0,t_{pred}))) || \\ &\mathrm{MLP}(\mathrm{Enc}(\mathrm{SD}(w;w_0,t_{pred})))
\end{split}
\end{equation}
where $||$ denotes the concatenation operation, and MLP denotes Multi-Layer Perceptions. 
The above defined STPE-U can characterize the spatial-temporal first-order proximity of the contextual node to the single target \textit{node}. To address the spatial-temporal high-order proximity between the target node pair $(u,v)$, the contextual nodes $w \in \mathcal{C}(u,t_{pred}) \cup \mathcal{C}(v,t_{pred})$ are leveraged as the auxiliary nodes for the target node pair $(u,v)$. Specifically, we propose the Correlated Spatial-Temporal Positional Encoding (STPE-C) of the contextual node $w$ as the combination of its STPE-U to both target nodes: 
\begin{equation}
\begin{split}
\label{eq::stpe-c}
    \mathrm{STPE\_c}(w; (u,v), t_{pred})= &\mathrm{STPE\_u}(w;u,t_{pred}) + \\ &\mathrm{STPE\_u}(w;v,t_{pred})
\end{split}
\end{equation}

\subsection{Network Architecture}
\label{subsection::network_architecture}
In this section, we present the network architecture of \model, which is a Transformer based model incorporating the structural information of CTDG. The input of \modell is the node embeddings of contextual nodes  $X=[x(w_1),...,x(w_C)] \in \mathbb{R}^{C \times d_{emb}}$ where $C=|\mathcal{C}(w_0,t_{pred})| (w_0 \in \{u,v\})$ is the size of contextual nodes set. The node embedding $x(w_i)$ is the concatenation of raw node feature of $w_i$ and its STPE-C with respect to the target link $(u,v)$ at time $t_{pred}$. 
In addition, some CTDG datasets may provide features associated with interaction events, which contains important semantic information about the correlations between the contextual node pairs. Therefore, we modify the self-attention module of the original Transformer to incorporate the event feature. Specifically, the event feature matrix is defined as $E=[e_{ij}]_{1\leq i,j \leq C} $, where $e_{ij}$ is the event feature if node $i$ and $j$ are interacted during contextual node sampling period otherwise a zero vector. The input node embedding $X$ is denoted as $H^{(0)}$. Then, the self-attention module of \modell (\model Attn) is defined as:
\begin{equation}
\label{eq:attn}
\begin{split}
   H^{(l)}_i &= {\rm \model Attn} (H^{(l-1)},E) \\
   &= \sum_{j=1}^L M_{ij} * {\rm Softmax}(A_{ij}^{(l)} (H_j^{(l-1)} W_V^{(l)}+e_{ij}W_{EV}^{(l)}) \\
    {\rm Where} \quad A_{ij}^{(l)} &= \dfrac{h_i^{(l-1)}W_Q^{(l)}(h_j^{(l-1)}W_K^{(l)}+e_{ij}W_{EK}^{(l)})^T}{\sqrt{d_k}}
\end{split}
\end{equation}
where $H^{(l)}_i$ denotes the $i$-th row of the matrix $H^{(l)}$. $W_Q^{(l)}, W_K^{(l)} \in \mathbb{R}^{d \times d_K}$, $W_V^{(l)} \in \mathbb{R}^{d \times d_V}, W_{EK} \in \mathbb{R}^{d_e \times d_K}, W_{EV}^{(l)} \in \mathbb{R}^{d_e \times d_V}$ are weight matrices. For simplicity, we set $d_K=d_V=d$ for the intermediate layers. In addition, $M \in \mathbb{R}^{C \times C}$ in Eq. (\ref{eq:attn}) is a masking matrix defined as follow:
\begin{eqnarray}
    M_{ij} =\left\{
    \begin{aligned}
        1 & , & t_i < t_j \quad {\rm and} \quad hop_i \geq hop_j \\
        0 & , & {\rm otherwise} \\
    \end{aligned}
    \right.
\end{eqnarray}
where $t$ and $hop$ are the timestamps and hop numbers obtained in contextual node sampling period. This masking matrix ensures that messages can only pass from the history to future, and from farther temporal neighbors to the closer temporal neighbors. 

Following the common practice of Transformer models, we adopt the the Layer Normalization (LN) \cite{ba2016layer} and residual connection \cite{he2016deep} in our \modell layer. For easier optimization, we adopt a Pre-Norm architecture \cite{xiong2020layer} where the Layer Normalization is applied before \model Attn and Feed-Forwad Networks (FFN). Formally, the \modell layer is defined as follows:
\begin{equation}
\label{eq::dgt_layer}
\begin{split}
     H'^{(l)} &= {\rm \model Attn}({\rm LN}(H^{(l-1)}),E)+H^{(l-1)} \\
     H^{(l)} &= {\rm FFN}({\rm LN}(H'^{(l)}))+H'^{(l)}\\
\end{split}
\end{equation}
where $L$ is the total number of layers. Multi-head self attention \cite{vaswani2017attention} can also be adopted to further enhance the expressive power of \model. The output of \modell layer $H^{(L)} \in \mathbb{R}^{C \times d}$ are the embeddings of contextual nodes. The embedding of the root node $w_0 (w_0 \in \{u,v\})$ at time $t_{pred}$, denoted as $z(w_0, t_{pred})$, is obtained by applying mean pooling on the node embeddings of its associated contextual nodes $\mathcal{C}(w_0,t_{pred})$:
\begin{equation}
\label{eq::root_node_embedding}
    z(w_0,t_{pred}) = \dfrac{1}{C}\sum_{i=1}^C H^{(L)}(i,:)
\end{equation}

\subsection{Training Objective}
 Given the target link $(u,v,t_{pred})$, the node embeddings $z(u,t_{pred})$ and $z(v,t_{pred})$ can be computed using Eq. (\ref{eq::root_node_embedding}). Then, the predicted score of $(u,v,t_{pred})$ is computed as:
 \begin{equation}
 \label{eq::merge_func}
     S(u,v,t_{pred}) = \sigma(\mathrm{MLP}(z(u,t_{pred}),z(v,t_{pred})))
 \end{equation}
 where $\sigma$ denotes the Sigmoid function.Finally, the Binary Cross Entropy (BCE) loss is adopted to train \model: 
\begin{equation}
\label{eq::loss}
\begin{split}
    \mathcal{L}=\sum_{(u,v,t)} &- \log S(u,v,t_{pred}) \\ &- \mathbb{E}_{ r \sim Unif(\mathcal{V} \backslash \{u,v\}}) \log (1-S(u,r,t_{pred}))
\end{split}
\end{equation}
where $Unif(\cdot)$ denotes a uniform sampling distribution on the node set. The overall training pipeline is in Appendix.

\subsection{Complexity Analysis}
In this section, we analyze the time and spatial complexity of the proposed \modell model. For time complexity, given a mini-batch of interactions of size $B$, sampling the contextual nodes requires to binary-search the insertion point of the timestamp and costs $O(B\log(\bar{d}))$, where $\bar{d}$ is the average degree of nodes. Computing the Temporal Distance costs $O(BC)$ complexity, where $C$ is the number of contextual nodes. Computing the Spatial Distance costs $O(K)$ complexity where $K$ is the maximum hop of contextual nodes. Forwarding the model costs $O(C^2HD)$ due to self-attention operation, where $H$ is the number of attention heads and $D$ is the hidden dimension of weights. Therefore, the time complexity of training \modell is $O(B(\log(\bar{d})+C)+C^2HD)$. For spatial complexity, storing the statistics of interactions cost $O(|\mathcal{E}|)$, where $|\mathcal{E}|$ is the total number of interactions. This spatial complexity is inevitable for learning CTDG models due to the CTDG data loading. The complexity comparison with other methods is presented in Appendix.

\section{Experiments}
\label{sec::exp}

\subsection{Experimental setup}

\label{sec::exp_setup}
\paragraph{\textbf{Datasets.}}
We evaluate the proposed model on nine Continuous-Time Dynamic Graph (CTDG) datasets: Reddit, Wikipedia, LastFM, UCI, Enron, Social Evolution, Flights, Gowalla-Food, and Gowalla-Outdoors. Among these, Reddit and Wikipedia constitute bipartite networks abundant in node/edge attributes, while LastFM represents a bipartite network devoid of node features. UCI, Enron, Social Evolution, and Flights are non-bipartite communication networks, also lacking attributes. Gowalla-Food and Gowalla-Outdoors are two large-scale datasets derived from the primary Gowalla dataset. For datasets without meaningful node features (i.e., LastFM, UCI, and Enron), we employ zero vectors as node features. Further details regarding these datasets are provided in Appendix.

\paragraph{\textbf{Baselines.}}

We compare the proposed \modell with three types of CTDG models:  (1) GNN-based: DyRep \cite{trivedi2019dyrep}, TGAT \cite{xuiclr2020tgat}, TGN \cite{rossi2020temporal} and Graphmixer \cite{cong2023we}. (2) Random walk based: CTDNE \cite{nguyen2018continuous} and CAW \cite{wangiclr2021caw}. (3) Transformer based: TCL \cite{wang2021tcl} and TGSRec \cite{fan2021continuous}.   More introductions about the baseline methods and tuned hyper-parameters are presented in Appendix.

\paragraph{\textbf{Evaluation Protocols.}}
Our evaluation protocols closely follow \cite{rossi2020temporal}. In specific, we adopt transductive/inductive link prediction and dynamic node classification tasks for evaluation. For \textit{Transductive} link prediction task, we split the total time range $[0,T]$ into three seperate intervals $[0, T_{train})$, $[T_{train}, T_{val}]$ and $[T_{val}, T)$ with $T_{train}/T=0.7$ and $T_{val}/T=0.85$ fixed. Then, we allocate the interactions happening within each interval to generate the training, validation and testing set. The inductive link prediction task follows the same splitting protocol as the transductive experiments. However, we randomly select 10\% of the nodes as "masking nodes", excluding any links associated with them in the training set, and removing any links not associated with them in the validation and testing sets. 

\paragraph{\textbf{Training Configurations.}}
We train all the models for 50 epochs and adopt the early stopping strategies. We adopt Adam optimizer and learning rate of 0.001 for all the tasks. Early stopping strategy is adopted. The batch size is set as 100. We sample 2-hop temporal neighbors for all datasets. More details about hyper-parameter of \modell and other baselines are presented in Appendix.


\begin{table*}[t]

\centering
\caption{The Average Precision (AP) results of transductive/inductive link prediction are reported. The values are multiplied by 100. The results of the best and second best performing models are highlighted in \textbf{bold} and \underline{underlined}, respectively.}
\label{table::link_prediction_ap}
\scalebox{0.9}{
\begin{tabular*}{\textwidth}{@{\extracolsep{\fill}}ccccccccc}
\toprule[1.0pt]
                               & Model     & Reddit        & Wikipedia            & LastFM        & UCI             & Enron  &  Social Evo. & Flights  \\
\midrule
\multirow{8}{*}{\rotatebox{90}{Transductive}} 
                               & JODIE    & 97.18$\pm$0.2 & 94.09$\pm$0.5  & 70.89$\pm$0.8 & 85.91$\pm$0.9 &  74.73$\pm$4.6 & 88.24$\pm$0.5 & 95.23$\pm$1.6 \\
                               & DyRep     & 98.09$\pm$0.1 & 94.54$\pm$0.3  & 68.45$\pm$3.2 & 52.94$\pm$0.8   & 69.36$\pm$3.4 & 87.93$\pm$0.3 & 94.97$\pm$0.6\\
                               & TGAT      & 98.15$\pm$0.1 & 94.69$\pm$0.1  & 54.42$\pm$0.9 & 77.75$\pm$0.2  & 58.91$\pm$0.3 & 92.15$\pm$0.2 & 94.76$\pm$0.2\\
                               & TGN  & 98.70$\pm$0.0 & 98.49$\pm$0.1  & 72.43$\pm$2.9 & 85.57$\pm$4.5 & 77.45$\pm$3.3 & \underline{93.27$\pm$0.1} & 97.23$\pm$0.1  \\
                               & CAWN & \underline{98.72$\pm$0.0} & \underline{98.79$\pm$0.1}  & \underline{85.52$\pm$0.3} & 92.51$\pm$0.1 & \underline{89.32$\pm$0.1} & 86.39$\pm$0.1 & \underline{98.12$\pm$0.3} \\
                               & TGSRec    & 88.56$\pm$2.4 & 85.68$\pm$1.0  & 67.60$\pm$3.9 & 76.00$\pm$0.5 &  69.46$\pm$2.8 & 74.51$\pm$0.7 & 94.18$\pm$0.3 \\
                               & TCL & 97.68$\pm$0.2  & 96.82$\pm$0.3  & 70.88$\pm$1.5  & 89.43$\pm$1.0  &   80.84$\pm$0.4 & 93.25$\pm$0.1 & 91.38$\pm$0.3  \\   
                               & Graphmixer &  97.42$\pm$0.0  &  97.25$\pm$0.1  &  78.26$\pm$0.4  &  \underline{93.46$\pm$0.3}  &  83.28$\pm$0.2 & 92.93$\pm$0.1 & 91.37$\pm$0.6 \\
\cmidrule{2-9}
                               & \modell      & \textbf{99.18$\pm$0.0} & \textbf{99.07$\pm$0.1}  & \textbf{92.23$\pm$0.1} & \textbf{96.03$\pm$0.2} &  \textbf{91.76$\pm$0.5}  & \textbf{93.81$\pm$0.2} & \textbf{98.81$\pm$0.0}\\
\midrule
  & Improve & 0.46 & 0.28 & 6.71 & 2.57 & 2.44 & 0.54 & 0.69 \\
\midrule
\multirow{10}{*}{\rotatebox{90}{Inductive}} 
                               & JODIE  & 94.46$\pm$0.1  & 92.98$\pm$0.1 & 82.88$\pm$1.3 & 72.93$\pm$0.7  &  72.97$\pm$2.0 & 91.49$\pm$0.5 & 95.17$\pm$0.4 \\
                               & DyRep     & 95.81$\pm$0.3 & 92.01$\pm$0.4  & 80.43$\pm$2.2 & 49.42$\pm$0.8 &  57.54$\pm$3.3 & 89.82$\pm$0.5 & 93.58$\pm$0.7 \\
                               & TGAT     & 97.18$\pm$0.6 & 93.50$\pm$0.1  & 55.63$\pm$1.7 & 70.96$\pm$0.7 &  58.44$\pm$2.4 & 90.74$\pm$0.4 & 89.46$\pm$0.3 \\
                               & TGN     & \underline{97.54$\pm$0.1} & \underline{97.83$\pm$0.1}  & 79.16$\pm$2.7 & 82.99$\pm$1.2 &  72.08$\pm$2.8 & 90.46$\pm$0.6 & 95.63$\pm$0.3 \\
                               & CAWN    & 97.19$\pm$0.6 & 97.33$\pm$0.4  & \underline{83.51$\pm$0.6} & 91.95$\pm$0.6  &  \underline{86.28$\pm$2.2} & 79.14$\pm$0.2 & \underline{96.80$\pm$0.2} \\
                               & TGSRec & 82.08$\pm$5.4 & 78.39$\pm$0.3 & 68.42$\pm$3.8 & 65.46$\pm$0.7  &  68.83$\pm$3.8 & 68.85$\pm$1.6 & 87.64$\pm$0.3 \\
                               & TCL & 93.89$\pm$0.5 & 97.08$\pm$0.2 & 72.56$\pm$1.7 & 88.37$\pm$1.8 &  77.28$\pm$0.8 & 91.67$\pm$0.1 & 83.41$\pm$0.1 \\
                               & Graphmixer  & 97.38$\pm$0.0 &  97.10$\pm$0.0 & 83.11$\pm$0.4 & \underline{92.10$\pm$0.6}  & 75.84$\pm$1.2 & \underline{92.03$\pm$0.1} & 83.03$\pm$0.1 \\
\cmidrule{2-9}
                               & \modell       & \textbf{98.82$\pm$0.1} & \textbf{98.48$\pm$0.2} & \textbf{93.03$\pm$1.8} & \textbf{94.70$\pm$1.0}  & \textbf{91.65$\pm$1.3} & \textbf{94.44$\pm$0.8} & \textbf{97.46$\pm$0.4}   \\
\midrule
& Improve & 1.28 & 0.65 & 9.52 & 2.60 & 5.37 & 2.41 & 0.66 \\
\bottomrule[1.0pt]
\end{tabular*}
}
\end{table*}

\subsection{Results and Discussion}

The Average Precision (AP) scores of transductive and inductive link prediction experiments are presented in Table \ref{table::link_prediction_ap}. The Area Under the receiver operating Characteristic (AUC) results are presented in Appendix. 
As can be seen from Table \ref{table::link_prediction_ap}, our \modell achieves the best AP performance on all the datasets for both transductive and inductive settings. Specifically, the transductive AP and inductive AP of \modell show an average improvement of 1.95\% and 3.21\%, respectively, demonstrating the effectiveness of \model. 


In addition, we make the following observations: (1) Our proposed model demonstrates robust performance across both attributed networks (Reddit and Wikipedia) and non-attributed datasets (LastFM, UCI, MOOC, and UCI). In contrast, the performance of several baseline methods, which lack node encodings designed for the evolving laws of CTDGs (such as Jodie, DyRep, and TGAT), significantly decreases on non-attributed datasets. This suggests the efficacy of our proposed STPE-C in modeling the evolving nature of CTDGs. (2) On the LastFM and Social Evolution datasets, our model improves the inductive Average Precision (AP) by 6.71\% and 9.52\% over the strongest baseline, respectively. This may be attributed to the significantly higher average interaction intensity of the Social Evolution ($\lambda=2.73 \times 10^{-3}$) and LastFM ($\lambda=5.04 \times 10^{-4}$) datasets compared to others. As such, the intensity may play a more crucial role in proximity estimation. Our proposed model excels at capturing long-term interaction intensity via temporal distance. (3) When compared to the parametric intensity-based method (DyRep), our proposed model displays significantly improved performance. This improvement can be attributed to the incorporation of parametric-free intensity into high-order proximity encoding. The results of node classification is presented in Appendix.

\subsection{Ablation Studies}

\begin{table*}[htbp]
\centering
\caption{Ablation studies results. Results of inductive link prediction are reported. The values are multiplied by 100. The best performance is marked in \textbf{bold}.}
\label{table::ablation_study}
\scalebox{0.9}{
\begin{tabular}{ccccccccc}
\toprule[1.0pt]
\multirow{2}{*}{Ablations}                                      & \multicolumn{2}{c}{Reddit} & \multicolumn{2}{c}{Wikipedia} & \multicolumn{2}{c}{LastFM} & \multicolumn{2}{c}{UCI} \\
\cmidrule(lr){2-3} \cmidrule(lr){4-5} \cmidrule(lr){6-7} \cmidrule(lr){8-9}
                                                                & AUC          & AP          & AUC            & AP           & AUC          & AP          & AUC         & AP        \\
\midrule
Full Model                                                      & \textbf{98.67$\pm$0.1}  &           \textbf{98.82$\pm$0.1} & \textbf{98.26$\pm$0.4}  & \textbf{98.48$\pm$0.2} &  \textbf{93.01$\pm$2.0} &  \textbf{93.03$\pm$1.8} & \textbf{93.20$\pm$1.8} &  \textbf{94.70$\pm$1.0}           \\
\midrule
1. Set $\alpha=0$ in Eq. (\ref{equation::TD2}) &  98.09$\pm$0.0   &     98.24$\pm$0.1 &  98.17$\pm$0.0
 &  98.28$\pm$0.3 & 89.10$\pm$0.6  &  92.07$\pm$0.4 &  91.30$\pm$0.3 &  91.52$\pm$0.3        \\
2. Set $\beta=0$ in Eq. (\ref{equation::TD2})  &  98.10$\pm$0.0     &     98.19$\pm$0.1 &  97.87$\pm$0.0
 &   98.18$\pm$0.2 &  89.55$\pm$0.6 &  90.80$\pm$1.7 & 90.46$\pm$0.3  &  94.28$\pm$0.8           \\
3. Remove temporal distance                                                               &  97.91$\pm$0.1   &  98.33$\pm$0.3 &  97.81$\pm$0.1  &   97.11$\pm$0.1 &   88.43$\pm$0.1 & 86.68$\pm$0.4 &  92.19$\pm$0.1
 & 91.20$\pm$0.3          \\
4. Remove spatial distance                                                               &   97.81$\pm$0.0   &      98.02$\pm$0.1 &  97.62$\pm$0.0  &  97.36$\pm$0.4 &  92.47$\pm$0.1  & 91.43$\pm$0.5 & 90.15$\pm$0.1 & 92.37$\pm$1.0          \\
5. Replace STPE-C with STPE-U                                                              &  96.99$\pm$0.2  &  93.87$\pm$1.1 &  89.22$\pm$2.2 &  75.38$\pm$5.1 &  64.82$\pm$0.7 &  50.03$\pm$2.9 &  78.59$\pm$2.2
 & 64.32$\pm$4.6          \\
6. Remove mask in Eq. (\ref{eq:attn})                                                               &   98.14$\pm$0.0   &   98.19$\pm$0.0 &  98.08$\pm$0.1  & 98.19$\pm$0.1 &  88.95$\pm$0.1 & 85.81$\pm$1.4 & 90.31$\pm$0.1
 &  92.05$\pm$0.3    \\
7. Use recent sampling  & 97.49$\pm$0.1   &  97.36$\pm$0.1   &  96.44$\pm$0.3  &  96.93$\pm$0.2  & 87.23$\pm$0.8  & 85.91$\pm$0.1 & 86.25$\pm$1.6  & 87.15$\pm$1.5  \\
\bottomrule[1.0pt]
\end{tabular}
}
\end{table*}

In this subsection, we conduct ablation studies to evaluate the effectiveness of different modules of \modell. The inductive AP and AUC results are shown in Table \ref{table::ablation_study}. The results are analyzed as follows: (1) In Ablations 1 and 2, we eliminate the recentness term and intensity term from the temporal distance calculation, respectively. As evidenced by the results, the removal of either term leads to a performance degradation across all datasets, especially on LastFM. This demonstrate the significance of both recentness and intensity in computing temporal distance. (2) In Ablations 3 and 4, we exclude the spatial distance and temporal distance from the STPE-C component, respectively. The results indicate that the performance across all datasets, particularly non-attributed ones, is compromised when either spatial distance or temporal distance is removed. This highlights the importance of modeling both spatial and temporal proximity in learning on CTDGs. (3)  In Ablation 5, we substitute the binary STPE-C with STPE-U as defined in Eq. \ref{eq::STPE_u}. In this scenario, the high-order proximity between the target nodes fails to be captured. We observe a significant drop in performance across all datasets, underscoring the critical role of modeling spatial-temporal high-order proximity in learning the evolving patterns of CTDGs. (4) In Ablation 6, we remove the masking matrix utilized in the Transformer model, which also leads to a performance decline across all datasets. (5) In Ablation 7, we replace the uniform contextual nodes sampling with the most recent sampling (i.e., the most recent interacted neighbors are sampled as the contextual nodes). The performance of \modell drops on all four datasets. The reason may be that the diversity of neighbors decrease when replaced with most recent sampling strategy, thus the high-order proximity estimation may be less accurate.

\subsection{Scalability Analysis}
\begin{figure}
    \centering
    \includegraphics[width=0.8\linewidth]{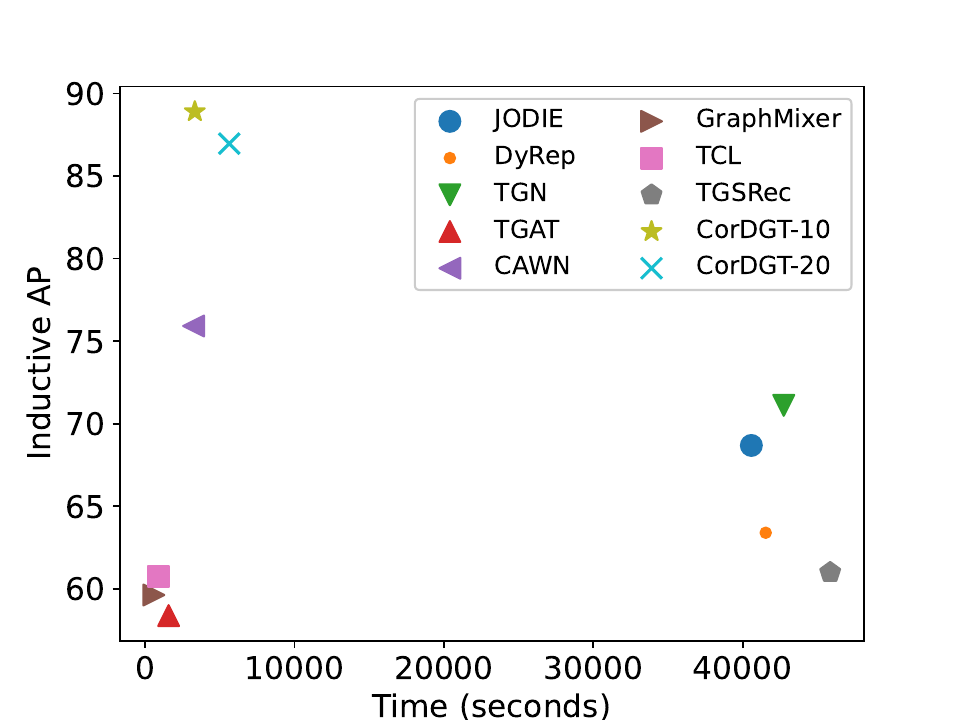}
    \caption{The inductive AP metrics and the training time per epoch on Food datasets. The closer to the upper left corner, the better performance. CorDGT-10 and CorDGT-20 denote the \modell with 10 and 20 contextual nodes, respectively.}
    \label{fig::gowalla-food-time}
\end{figure}

In this section, we evaluate the performance and efficiency of the proposed \modell on large-scale datasets. We adopt \textit{Outdoors} and \textit{Food} subsets from the large Gowalla \cite{liu2014exploiting} dataset for evaluation. The Outdoors dataset contains around 0.22M nodes and 1.19M edges. The Food dataset has around 0.67M nodes and 2.71M edges. We run all the models for one epoch and compare the performance. The inductive AP metrics and the training speed per epoch of the Food dataset are presented in Figure \ref{fig::gowalla-food-time}. The results of Outdoors are presented in Appendix. 

As can be seen from Figure \ref{fig::gowalla-food-time}, both two configurations of \modell can achieve consistently outperform the baselines in terms of inductive AP. In addition, training one epoch of \model with 10 contextual neighbors takes 3319 seconds, which is significantly faster than memory-based models (JODIE, DyRep and TGN). It is because the memory-based models need to store the memory state for each node. Given the vast number of nodes in large-scale datasets, the process of storing and managing the memory state incurs a considerable computational cost.  Our proposed \modell is marginally slower than CAWN (3234 seconds), but has significant better inductive AP performance.

\subsection{Visualization}

\begin{figure}
    \centering
    \includegraphics[width=0.9\linewidth]{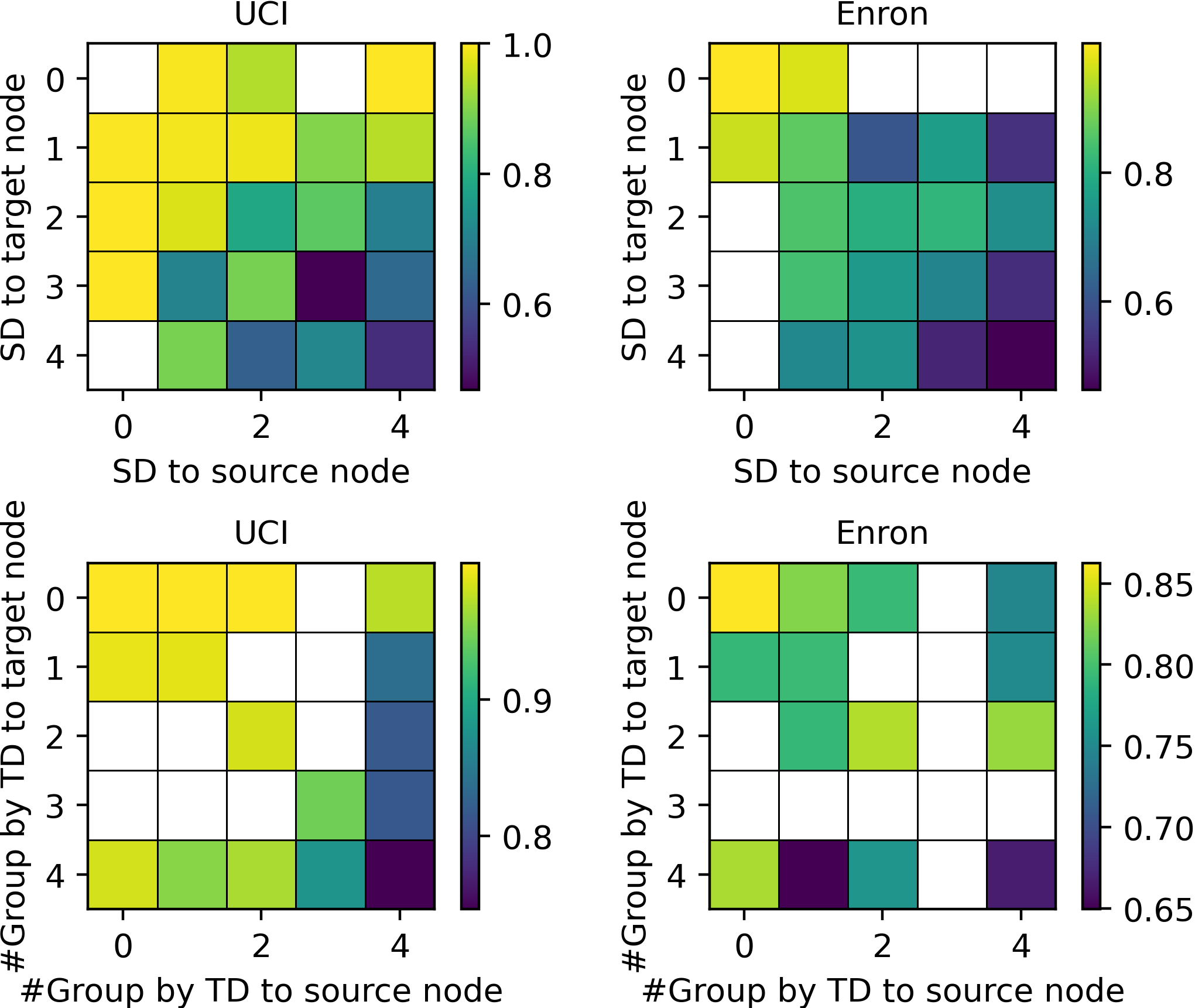}
    \caption{Heatmap values indicate the confidence score on positive source/target node pairs predicted by different groups of contextual nodes. Left to right: UCI and Enron datasets. Top to bottom: the contextual nodes are grouped according to their Spatial Distance and Temporal Distance to source/target nodes. The blank cells indicate that no data is allocated to this group. Best viewed in color. }
    \label{fig::visualize}
\end{figure}

One motivation of the proposed \modell is to capture the spatial-temporal high-order proximity in CTDGs by considering the distance of the contextual nodes to both ends of target nodes. 
To further demonstrate the interpretability of the proposed \model, we visualize the link prediction score by the contextual nodes with different Temporal Distance (TD) and Spatial Distance (SD) to the target nodes. 
Specifically, we replace Eq. (\ref{eq::merge_func}) as: 
\begin{equation}
    S(u,v,t_{pred}) = \sigma(\Phi^T(z(u,t_{pred})+z(v,t_{pred})))
\end{equation}
where $\Phi: \mathbb{R}^d \rightarrow \mathbb{R}^1$ is a trainable linear projector, and train the model.
In this way, the contribution of each contextual node $w$ to the link prediction score can be decomposed as $\Phi^T(h_w)$. 
After training, we randomly select a mini-batch of contextual nodes. To see the influence of temporal distance to the predicted score, we evenly split the range of TD to target node pairs in this mini-batch to 5 groups, which formulates total 25 buckets. Then, we allocate the contextual nodes into these buckets based on their TD to target node pairs. Similarly, we can allocate the contextual nodes into different buckets based on their SD to target node pairs. Finally, we compute the average prediction score of the contextual nodes in each bucket. The visualization results on UCI and Enron are shown in Figure \ref{fig::visualize}. As can be seen from Figure \ref{fig::visualize}, the contextual nodes have closer temporal distance (closer to top-left corner) to both ends of the target link will give higher prediction score of the link (closer to yellow). This indicates that the proposed \modell will give higher prediction if the contextual nodes have smaller temporal distance to target node pairs. Similar observations are also seen from spatial distance. Therefore, the proposed \modell may be capable to capture the spatial-temporal high-order proximity. 



\section{Conclusions}
This paper introduces the Dynamic Graph Transformer with Correlated Spatial-Temporal Positional Encoding (\model), a novel approach for representation learning on Continuous-Time Dynamic Graphs (CTDGs). We employ the Poisson Point Process assumption and sampled temporal neighborhood to achieve comprehensive proximity estimation on CTDGs. Subsequently, we propose Correlated Spatial-Temporal Positional Encodings (STPE-C), which utilizes the comprehensive proximity to capture spatial-temporal high-order proximity. Extensive experiments conducted on seven small and two large-scale datasets demonstrate the performance superiority and scalability of the proposed \modell model. A potential future direction for this work could involve designing more sophisticated spatial-temporal distances for improved proximity estimation and preservation. 

\section*{ACKNOWLEDGMENTS}
This work is supported by the National Natural Science Foundation of China (62476244), Zhejiang Provincial Natural Science Foundation of China (Grant No: LTGG23F030005), National Natural Science Foundation of China (62372399,62476245) and the advanced computing resources provided by the Supercomputing Center of Hangzhou City University.

\section*{Ethical Considerations}
The proposed \modell is used to learn the temporal embeddings which can be leveraged for downstream tasks such as link prediction and node classification. The direct negative societal effects of this research, encompassing fairness, privacy, and security considerations, are minimal. Nevertheless, akin to other predictive models, a few erroneous predictions by the model could impact system functionality. Despite extensive experimental validation of the model's efficacy, occasional inaccurate predictions, particularly on outlier data, remain plausible. Therefore, enhancing data quality through measures such as data cleaning prior to model application is advised.

\bibliographystyle{ACM-Reference-Format}
\bibliography{ref}

\appendix

\newpage

\section{Overall Algorithm of \model}
\label{sec::appendix_algorithm}

\begin{algorithm}
\caption{CorDGT Training}
\label{alg::train}
\begin{algorithmic}[1]
\REQUIRE  Training data $\mathcal{E}=[(u_i,v_i,t_i,e_i)]_{i=1}^M$, node feature $\mathcal{X}=[x(i)]_{i=1}^N$, hop number $K$, sampling number of contextual nodes $\{n_i\}_{i=1}^K$, $\alpha$, $\beta$.            
\STATE Initialize sparse matrices $\mathcal{M}_n$ and $\mathcal{M}_t$.
\STATE $loss \leftarrow 0$
\FOR {$(u,v,e,t) \in \mathcal{E}$}
\STATE $r \leftarrow$ Sampled negative node;
\STATE Sample the contextual node sets $\mathcal{C}(u,t), \mathcal{C}(v,t), \mathcal{C}(r,t)$.
    \FOR{$(w,t') \in \mathcal{C}(u,t) \cup \mathcal{C}(v,t)$} 
    \STATE Read the interaction counts $cnt(u,w)$ and $cnt(v,w)$ from $\mathcal{M}_n$; Read the latest interaction timestamps $t_n(u,w)$ and $t_n(v,w)$ from $\mathcal{M}_t$. 
    \STATE Compute $\mathrm{STPE\_c}(w,(u,v),t_{pred})$ using Eq. (\ref{eq::stpe-c}).
    \STATE $h^{(0)}(w) \leftarrow \mathrm{STPE\_c}(w,(u,v),t_{pred}) || x(w)$.
    \ENDFOR
    \STATE $H^{(0)}_{pos}(u) \leftarrow \mathrm{stack}(\{h^{(0)}(w))|w \in \mathcal{C}(u,t)\})$, $H^{(0)}_{pos}(v) \leftarrow \mathrm{stack}(\{h^{(0)}(w))|w \in \mathcal{C}(v,t)\})$
    \STATE Forward $H^{(0)}_{pos}(u)$ and $H^{(0)}_{pos}(v)$ using Eq. (\ref{eq::dgt_layer}-\ref{eq::root_node_embedding}) to obtain $z_{pos}(u,t)$ and $z_{pos}(v,t)$.
    \STATE Forward $z_{pos}(u,t)$ and $z_{pos}(v,t)$ using Eq. (\ref{eq::merge_func}) to obtain score $S(u,v,t)$.
    \FOR{$(w,t') \in \mathcal{C}(u,t) \cup \mathcal{C}(r,t)$}
    \STATE Read the interaction counts $cnt(u,w)$ and $cnt(v,w)$ from $\mathcal{M}_n$; Read the latest interaction timestamps $t_n(u,w)$ and $t_n(v,w)$ from $\mathcal{M}_t$. 
    \STATE Compute $\mathrm{STPE\_c}(w,(u,r),t_{pred})$ using Eq. (\ref{eq::stpe-c}).
    \STATE $h^{(0)}(w) \leftarrow \mathrm{STPE\_c}(w,(u,r),t_{pred}) || x(w)$.
    \ENDFOR
    \STATE $H^{(0)}_{neg}(u) \leftarrow \mathrm{stack}(\{h^{(0)}(w))|w \in \mathcal{C}(u,t)\})$, $H^{(0)}_{neg}(r) \leftarrow \mathrm{stack}(\{h^{(0)}(w))|w \in \mathcal{C}(r,t)\})$
    \STATE Forward $H^{(0)}_{neg}(u)$ and $H^{(0)}_{neg}(r)$ using Eq. (\ref{eq::dgt_layer}-\ref{eq::root_node_embedding}) to obtain $z_{neg}(u,t)$ and $z_{neg}(r,t)$.
    \STATE Forward $z_{neg}(u,t)$ and $z_{neg}(r,t)$ using Eq. (\ref{eq::merge_func}) to obtain score $S(u,r,t)$.
    \STATE $loss \leftarrow loss + BCE(S(u,v,t), S(u,r,t))$
    \STATE Update $\mathcal{M}_n(u,v) \leftarrow \mathcal{M}_n(u,v)+1$, $\mathcal{M}_t(u,v) \leftarrow t$ 
\ENDFOR
\RETURN $loss$
\end{algorithmic}
\end{algorithm}

\section{More discussion on other Stochastic process}
The intensity is the number of happening times within a certain time interval of a counting process. Therefore, non-counting processes, such as Gaussian process and Wiener process, are not applicable for intensity estimation. Poisson process is the most well-known and commonly used counting process. It assumes the intensity is a constant, which provides great mathematical tractability and computation efficiency. Most of other counting processes are generalized from Poisson processes, such as non-homogeneous Poisson process, Hawkes process and Markovian arrival process, which includes more undetermined parameters for flexibility. Estimating the parameters of these process requires significantly higher computation budget than Poisson process.

\section{Complexity Comparison with existing models}
In Section 4, we analyze the computation complexity of \modell. In this section, we also analyze the computation complexity of two baseline methods, i.e., TGAT \cite{xuiclr2020tgat} and TGN \cite{rossi2020temporal}. For these models, given a mini-batch of interactions of size $B$, the historical neighbor sampling process costs $O(B\log(\bar{d}))$ , where $\bar{d}$ represents the average degree of nodes. The forward process, costs $O(CHD)$ where $C,H,D$ denote the number of sampled neighbors, attention heads, and hidden dimensions, respectively. Thus, the overall time complexity of TGAT and TGN is $(Blog(\bar{d})+CHD)$. The size of network parameters is denoted as $P$. In addition, TGN requires a memory and costs the spatial complexity of $O(ND')$ where $N$ and $D'$ denote the number of nodes and hidden dimension of the memory. We summarize the time and space complexity in Table \ref{table::computation_complexity}. 

\begin{table}
\caption{The time and space complexity of different models.}
\label{table::computation_complexity}
\centering
\begin{tabular}{ccc}
\toprule[1.0pt]
 & Space Complexity   &  Time Complexity \\
\midrule
TGAT &  $O(B\log(\bar{d})+CHD)$	    &   $O(P)$     \\
TGN &  $O(B\log(\bar{d})+CHD)$   &    $O(P+ND')$    \\
\modell &  $O(B(\log(\bar{d})+C)+C^2HD)$   &  $O(P+|\mathcal{E}|)$   \\
\bottomrule[1.0pt]
\end{tabular}
\end{table}

\section{Experimental Setting}
\label{sec::experiment_setting}

\subsection{Datasets}
\label{subsection:datasets}

\begin{table*}[htbp]
\centering
\caption{Dataset statistics. Average Interaction Intensity $\lambda = 2|\mathcal{E}|/|\mathcal{V}|\mathcal{T}$ \cite{wangiclr2021caw}, where $|\mathcal{E}|$ and $|\mathcal{V}|$ denote the number of interactions and nodes, respectively. $\mathcal{T}$ denotes the total duration (seconds). }
\label{table::dataset statistics}
\scalebox{0.9}{
\begin{tabular}{cccccccccc}
\toprule[1.0pt]
                                        & Reddit          & Wikipedia      & LastFM          & UCI                  & Enron    & Social Evolution & Flights & Outdoors & Food  \\
\midrule 
\# Nodes      & 10,984 & 9,227  & 1,980  & 1,899 &   184 & 74 & 13,169 & 223,777 & 673,858 \\
\# Links  & 672,447 & 157,474 & 1,293,103 & 59,835 & 125,235 & 2,099,519 & 1,927,145 & 1,192,397 & 2,708,688 \\
\# Nodes attributes      & 172      & 172     & 0          & 0            & 0   & 0  & 0 & 0 & 0  \\
\# Link attributes & 172 & 172 & 0 & 0 & 0 & 2 & 1 & 0 & 0 \\

$\lambda$ &      $4.57\times10^{-5}$           &     $1.27\times10^{-5}$       &    $5.04\times10^{-4}$             &       $3.79\times10^{-6}$      &       $1.20\times10^{-5}$  &  $2.73\times10^{-3}$    &  $1.41\times10^{-5}$ & $1.71\times10^{-7}$ & $2.58\times10^{-7}$ \\
Is bipartite?                           &     True            & True &            True     &      False                    & False   &  False  &  False  & True & True  \\
\bottomrule[1.0pt]
\end{tabular}
}
\end{table*}

Our experiments section includes seven public datasets: Reddit \footnote{\url{http://snap.stanford.edu/jodie/reddit.csv}}, Wikipedia\footnote{\url{http://snap.stanford.edu/jodie/wikipedia.csv}}, UCI\footnote{\url{http://konect.cc/networks/opsahl-ucsocial/}}, LastFM\footnote{\url{http://snap.stanford.edu/jodie/lastfm.csv}}, Enron\footnote{\url{https://www.cs.cmu.edu/~enron/}}, Social Evolution\footnote{\url{http://realitycommons.media.mit.edu/socialevolution.html}} and Flights\footnote{\url{https://zenodo.org/records/3974209\#.Yf62HepKguU}}. Reddit network is an user action datasets which consists of subreddits posted by different users in one month on Reddit website. It is a bipartite dataset consisting of 10000 most active users and 984 subreddits with rich interaction feature provided. Wikipedia network records the clicking actions on wikipedia pages by different users. It is a bipartite network consisted by clicking actions on 1000 pages in one month made by users with rich interaction feature provided. UCI network is non-bipartite network which contains sent messages between the users of an online community of students from the University of California, Irvine. The nodes represent students and the edges represent the communicated messages among them. Enron is a non-bipartite dataset which consists of approximately 0.5M emails that were exchanged between employees of the Enron energy company over a span of three years. Social Evolution is a mobile phone proximity network which tracks the everyday life of a whole undergraduate dormitory from October 2008 to May 2009. Flights is a directed dynamic flight network illustrating the development of the air traffic during the COVID-19 pandemic, which was extracted and cleaned for the purpose of this study. Each node represents an airport and each edge is a tracked flight. The edge weights specify the number of flights between two given airports in a day. In addition, we select the large-scale Gowalla \cite{liu2014exploiting} dataset for scalability evaluation. Gowalla is a social network for users check-ins at various locations, containing about 36 million check-ins made by 0.32 million users over 2.8 million locations. These check-in records are in the time span of Jan 2009 - June 2011.
The locations are grouped into 7 main fields. We select a subset of Outdoors and Food field for experiments. Specifically, we choose the part of the Outdoors data from Jan. 2009 to Dec. 2010 and the part of the Food data from Jan. 2011 to June 2011. 
Detailed statistics of aforementioned datasets are presented in Table \ref{table::dataset statistics}.

\subsection{Baselines}
\label{sec::baseline_intro}
The brief introduction of baseline methods in the Experiments section are as follows: 

\paragraph{\textbf{CTDNE \cite{nguyen2018continuous}}} CTDNE extends DeepWalk \cite{perozzi2014deepwalk} to dynamic graphs which leverages a SkipGram model on the temporal random walk sequence and learn the node embeddings. 
\paragraph{\textbf{JODIE \cite{kumar2019predicting}}} Jodie updates node embeddings in an interaction via two coupled RNNs, which are leveraged for future link prediction via a temporal projector.
\paragraph{\textbf{DyRep \cite{trivedi2019dyrep}}} DyRep updates the node embeddings involved in an interaction by a recurrent model considering the messages from 2-hop temporal neighbors.
\paragraph{\textbf{TGAT \cite{xuiclr2020tgat}}} TGAT extends GAT \cite{velickovic2018graph} and GraphSAGE \cite{hamilton2017inductive} to dynamic graphs, which samples and recursively aggregates the messages of k-hop temporal neighbors. The temporal representation is obtained by Fourier transformation on the time interval.  
\paragraph{\textbf{TGN \cite{rossi2020temporal}}} TGN proposes a generalized message-passing networks by extending Jodie and TGAT with a per-node memory mechanism for long-time interactions. 
\paragraph{\textbf{TGSRec \cite{fan2021continuous}}} TGSRec proposes a Temporal Collaborative Transformer which simultaneously captures the collaborative signals from users and items as well as temporal dynamics. 
\paragraph{\textbf{TCL \cite{wang2021tcl}}} proposes a dynamic-graph-topology-aware Transformer with a two-stream encoder for semantic inter-dependency modeling. Contrastive learning is adopted to maximize mutual information between future interaction nodes. 
\paragraph{\textbf{CAWN \cite{wangiclr2021caw}}} CAWN samples temporal random walks and \\ anonymize the node identities via Causal Anonymous Walks (CAW). The node encodings on temporal random walks are learned via a sequential model. 
\paragraph{\textbf{Graphmixer \cite{cong2023we}}} Graphmixer proposes a conceptually simple architecture that leverage MLP and mean-pooling to aggregate the temporal information and node features of K most recently interacted neighbors. Note that the original Graphmixer leverages one-hot node encoding as input, thus can not be applied for inductive experiment. In this work, we replace the one-hot encodings as the node encodings used by other baselines for fair comparison.

\subsection{Hyper-parameters Tuning}
\label{sec::hyper-parameter}
For all the baselines, we set the dimension of time encoding and hidden unit as 100 and 172, respectively. For Jodie, DyRep and TGN, we adopt the implementation \footnote{\url{https://github.com/twitter-research/tgn}} for evaluation. In specific, we set the memory dimension as 32 (for UCI) or 172 (for Reddit, Wikipedia and LastFM). We adopt a one-layer model with 10 temporal neighbors being sampled.  The official implementations of TGAT  \footnote{\url{https://github.com/StatsDLMathsRecomSys/Inductive-representation-learning-on-temporal-graphs}} and TGSRec \footnote{\url{https://github.com/DyGRec/TGSRec/}} are adopted. The number of layers, attention heads and the sampled temporal neighbors are set as 2, 2 and 20, respective.

For CAWN, we adopt its official implementation \footnote{\url{https://github.com/snap-stanford/CAW}}. We grid search following hyper-parameters: the time scaling factor is set as $\{10^{-4},10^{-5},10^{-6},10^{-7}\}$, the random walk length in $\{2,3,4,5\}$ and the number of walks in $\{16,32,64,128\}$. 

For Graphmixer, we adopt the original implementation \footnote{\url{https://github.com/CongWeilin/GraphMixer}}. The time gap is set as 2000. The number of MLP-Mixer layers is set as 2. For TCL, we adopt the implementation of DyGLib \footnote{\url{https://github.com/yule-BUAA/DyGLib}}. The number of layers, attention heads and the sampled temporal neighbors are set as 2, 2 and 20, respective.

\begin{table*}
\centering
\caption{Area Under ROC curve (AUC) results of transductive/inductive link prediction. The values are multiplied by 100. The results of the best and second best performing models are highlighted in \textbf{bold} and \underline{underlined}, respectively.}
\label{table::link_prediction_auc}
\begin{tabular*}{\textwidth}{@{\extracolsep{\fill}}ccccccccc}
\toprule[1.0pt]
                               & Model     & Reddit        & Wikipedia            & LastFM        & UCI       & Enron & Social Evo. & Flights            \\
\midrule
\multirow{8}{*}{\rotatebox{90}{Transductive}}                 
& JODIE     & 97.51$\pm$0.2  & 94.67$\pm$0.1  & 70.43$\pm$0.6  &  88.67$\pm$0.2 &     79.65$\pm$2.6 &  91.74$\pm$0.5 & 96.17$\pm$1.2 \\
& DyRep     &  98.03$\pm$0.1 &  94.05$\pm$0.4  & 68.77$\pm$2.9  &  55.97$\pm$1.7  &   73.21$\pm$2.2 & 91.02$\pm$0.2  &  96.84$\pm$0.5  \\
& TGAT      & 98.04$\pm$0.1  &  94.42$\pm$0.1  & 51.81$\pm$1.2  & 79.73$\pm$0.1  &    59.25$\pm$0.2 & 94.37$\pm$0.2  & 94.24$\pm$0.2   \\
& TGN-attn  &  \underline{98.67$\pm$0.0} & 98.42$\pm$0.1  &  72.48$\pm$2.8 &  86.35$\pm$3.1   &    79.54$\pm$2.6 & \underline{94.93$\pm$0.2}  & 98.06$\pm$0.1   \\
& CAWN-attn & 98.63$\pm$0.0  &  \underline{98.63$\pm$0.1} & \underline{82.36$\pm$0.5}  &  89.86$\pm$0.3   &      \underline{90.28$\pm$0.2} & 88.21$\pm$0.1  &  \underline{98.30$\pm$0.0} \\
& TGSRec    & 88.96$\pm$1.5  &  85.28$\pm$0.7  &  66.67$\pm$3.7  & 71.37$\pm$0.6   &   71.85$\pm$2.1 & 77.47$\pm$0.0  &  95.04$\pm$0.3  \\
& TCL & 97.67$\pm$0.0 & 96.07$\pm$0.2 & 65.23$\pm$1.2 & 87.85$\pm$1.4 &   76.47$\pm$0.8 & 93.87$\pm$0.2 &  91.37$\pm$0.5 \\
& Graphmixer    & 97.53$\pm$0.2 & 97.14$\pm$0.1  & 75.28$\pm$0.1  &  \underline{90.63$\pm$0.2}  &     82.76$\pm$0.4 & 94.61$\pm$0.1  & 91.32$\pm$0.0  \\
\cmidrule{2-9}
& \modell   & \textbf{99.18$\pm$0.1}   &  \textbf{98.84$\pm$0.1}  & \textbf{92.37$\pm$0.2}   & \textbf{95.12$\pm$0.4} &     \textbf{93.28$\pm$0.4} & \textbf{95.48$\pm$0.1} & \textbf{98.84$\pm$0.0}  \\
\midrule
 & Improve &  0.51 &  0.21  &  10.01  & 4.84  &     3.00 & 0.55 & 0.54 \\
\midrule
\multirow{8}{*}{\rotatebox{90}{Inductive}}             
& JODIE  & 95.02$\pm$0.2  & 92.65$\pm$0.3  & 82.18$\pm$1.2  & 74.74$\pm$1.5   &     75.65$\pm$2.6 & 93.78$\pm$0.3 & 95.36$\pm$0.3  \\
& DyRep  & 95.81$\pm$0.3  &  91.15$\pm$0.7 & 79.81$\pm$2.0  & 48.63$\pm$1.3  &     56.93$\pm$2.0 & 91.43$\pm$0.5 & 93.64$\pm$0.9 \\
& TGAT   & 97.06$\pm$0.7  & 93.08$\pm$0.1  & 53.50$\pm$2.4  & 70.89$\pm$0.4  &      58.18$\pm$3.0 & 93.26$\pm$0.6 & 89.35$\pm$0.4\\
& TGN-attn   &  \underline{97.45$\pm$0.1}  &  \underline{97.74$\pm$0.1} & 78.61$\pm$2.9  & 82.08$\pm$2.7  &     72.15$\pm$2.3 & 93.54$\pm$0.6 & 96.03$\pm$0.4  \\
& CAWN-attn  & 97.23$\pm$0.4 &  96.43$\pm$0.4 & \underline{88.67$\pm$0.6} &  88.91$\pm$0.8  &     \underline{87.24$\pm$0.5} & 84.57$\pm$0.3  & \underline{96.64$\pm$0.0} \\
& TGSRec &  81.99$\pm$4.2 & 79.30$\pm$0.7  &  67.93$\pm$3.0  &  61.66$\pm$1.5  &     72.32$\pm$2.2 & 63.57$\pm$1.5 & 90.31$\pm$0.1 \\
& TCL &  93.84$\pm$0.1 & 95.36$\pm$0.2 & 71.32$\pm$1.0 & \underline{92.14$\pm$0.1} &   73.28$\pm$1.1 & 93.85$\pm$0.2 & 84.54$\pm$2.0 \\
& Graphmixer     & 94.65$\pm$0.0 & 97.38$\pm$0.0  &  81.65$\pm$0.1  & 85.79$\pm$0.1  &      76.38$\pm$0.7 & \underline{94.13$\pm$0.1} & 81.98$\pm$1.7  \\
\cmidrule{2-9}
& \modell    & \textbf{98.67$\pm$0.8}  & \textbf{98.26$\pm$0.3}  & \textbf{93.01$\pm$1.8}  &  \textbf{93.20$\pm$0.9}  &    \textbf{92.64$\pm$1.9} & \textbf{95.69$\pm$1.0} & \textbf{98.12$\pm$0.4} \\
\midrule
 & Improve &  1.24 & 0.52 & 4.34 & 1.06  &    5.40 & 1.56 & 1.48 \\
\bottomrule[1.0pt]
\end{tabular*}
\end{table*}

\subsection{Implementation details of \modell}
For fair comparison, we train all the models for 50 epochs with early stopping executed if there is no improvement on validation AP for 3 epochs. In addition, we use the batch size of 100 for all models. We repeat the methods for 3 runs and report the mean and standard deviation of statistics. For the proposed \model, we set the learning rate as 0.001 and optimizer as Adam for all datasets. The attention heads, layer number and hidden dimension of main \modell encoder is set as 6, 2, 64, respectively, for all datasets. We sample 2-hop temporal neighbors for all datasets. For LastFM and UCI, the sampling numbers of contextual neighbors are \{32,1\}, and for other datasets, the sampling numbers of contextual neighbors are \{20,1\}. We set $\alpha=0.1$ and $\beta=1.0$ for UCI and $\alpha=1$ and $\beta=0.1$ for LastFM. We set $\alpha=1$ and $\beta=10$ for other datasets. We set the dimension of Correlated Spatial-Temporal Positional Encoding as 200, where the dimension of Spatial Distance encoding and Temporal distance is set as 100 and 100, respectively. All the experiments are run on a Linux Ubuntu 18.04 Server with a NVIDIA RTX2080Ti GPU.

\section{Additional Experimental results}

\subsection{Node classification results}
\label{sec::node_cls}

\begin{table}
\caption{The AUC results of node classification. The values are multiplied by 100. The baseline results are taken from \cite{rossi2020temporal}. The best results are highlighted in \textbf{bold}. }
\label{table::node_classification}
\centering
\begin{tabular}{ccc}
\toprule[1.0pt]
Method & Wikipedia     & Reddit \\
\midrule
CTDNE  & 75.89$\pm$0.5 &  59.43$\pm$0.6      \\
JODIE  & 84.84$\pm$1.2 &  61.83$\pm$2.7      \\
DyRep  & 84.59$\pm$2.2 &  62.91$\pm$2.4      \\
TGAT   & 83.69$\pm$0.7 &  65.56$\pm$0.7      \\
TGN    & 87.81$\pm$0.3 &  67.06$\pm$0.9      \\
\midrule
\modell    & \textbf{88.16$\pm$0.3} &  \textbf{70.33$\pm$0.6}     \\
\bottomrule[1.0pt]
\end{tabular}
\end{table}

The experiments settings of node classification are as follows. We initially train a model using the transductive link prediction task. Subsequently, we load the trained model and freeze its parameters, then append a classifier on top of it for the purpose of classification. The AUC results of dynamic node classification are presented in Table \ref{table::node_classification}. Note that we replace STPE-C with STPE-U as the positional encodings, since the node classification task only considers single node rather than interaction. Our \modell achieves the best performance on both Wikipedia and Reddit datasets compared with other baselines. 

\subsection{AUC results of link prediction}

The AUC results of link prediction are presented in Table \ref{table::link_prediction_auc}. As can be seen, our proposed \modell consistently outperforms other baselines on both transductive and inductive experiments.

\subsection{Scalability on Gowalla Outdoors}
\label{sec::gowalla_outdoor_time}
\begin{figure}
    \centering
    \includegraphics[width=0.8\linewidth]{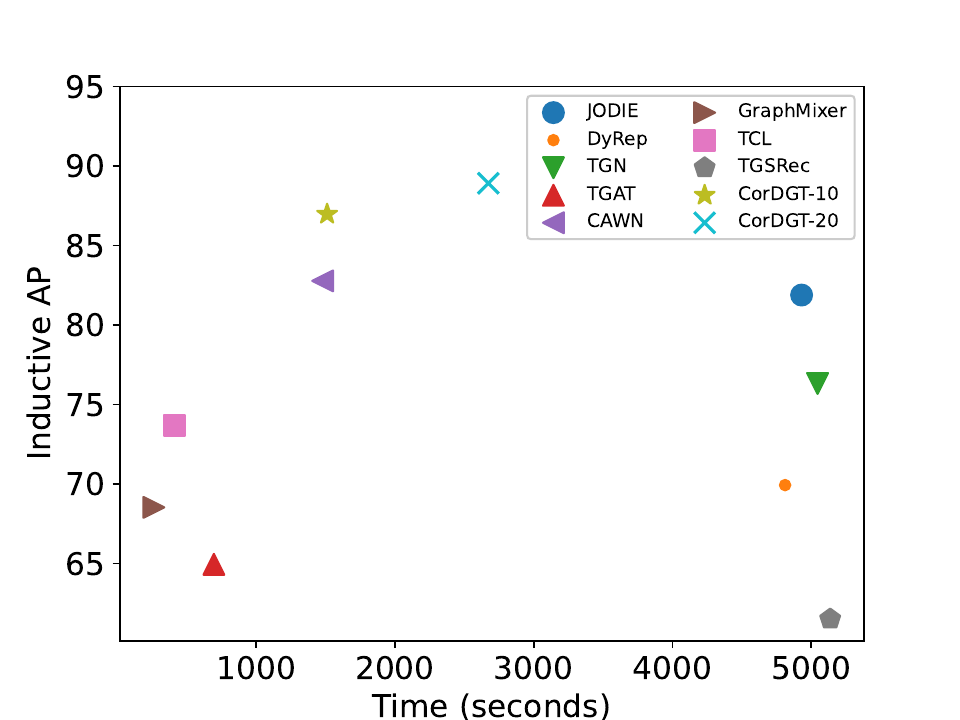}
    \caption{The inductive AP metrics and the training time per epoch on Outdoors datasets. The closer to the upper left corner, the better performance. CorDGT-10 and CorDGT-20 denote the \modell with 10 and 20 contextual nodes, respectively.}
    \label{fig::gowalla-outdoors-time}
\end{figure}

We further evaluate the scalability of the proposed \modell on Gowalla-Food datasets. We train all the models one epoch. The inductive AP and training time per epoch results are presented in Figure \ref{fig::gowalla-outdoors-time}. We observe that \modell with 10 or 20 contextual nodes can obtain the highest inductive AP among baselines. In addition, the training speed of \modell with 10 contextual nodes (1511 seconds per epoch) is significantly faster than TGN (5046 seconds per epoch) on Food datasets. This result demonstrate the scalability of the proposed \modell. 

\subsection{Parameters Sensitivity}

\begin{figure}
  \centering
  \begin{minipage}[b]{0.45\linewidth}
    \centering
    \subfigure[]{\includegraphics[width=\linewidth]{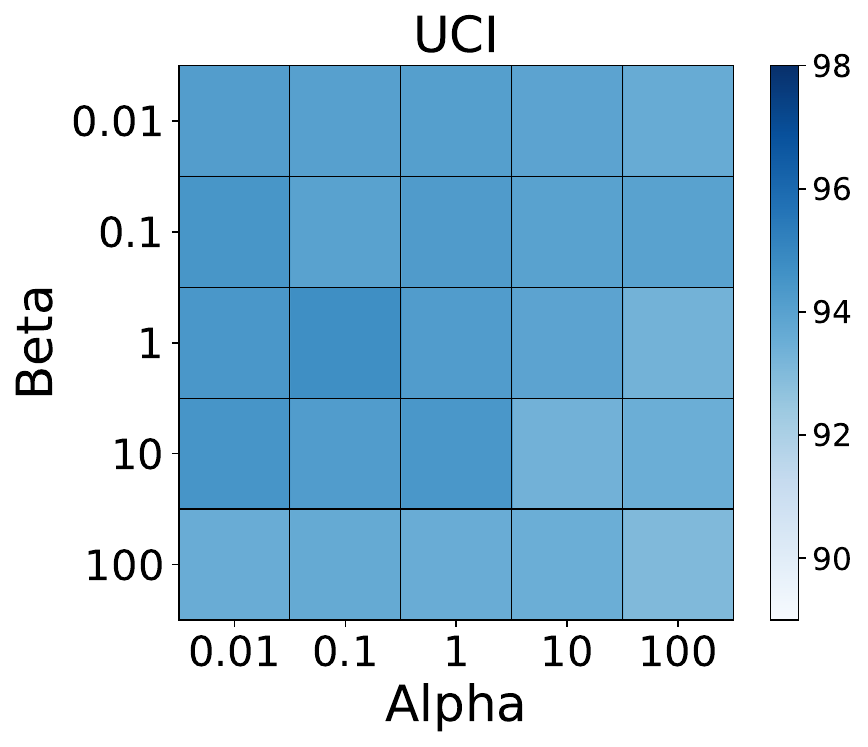}\label{fig:uci_alpha_beta}}
  \end{minipage}
  \begin{minipage}[b]{0.45\linewidth}
    \centering
    \subfigure[]{\includegraphics[width=\linewidth]{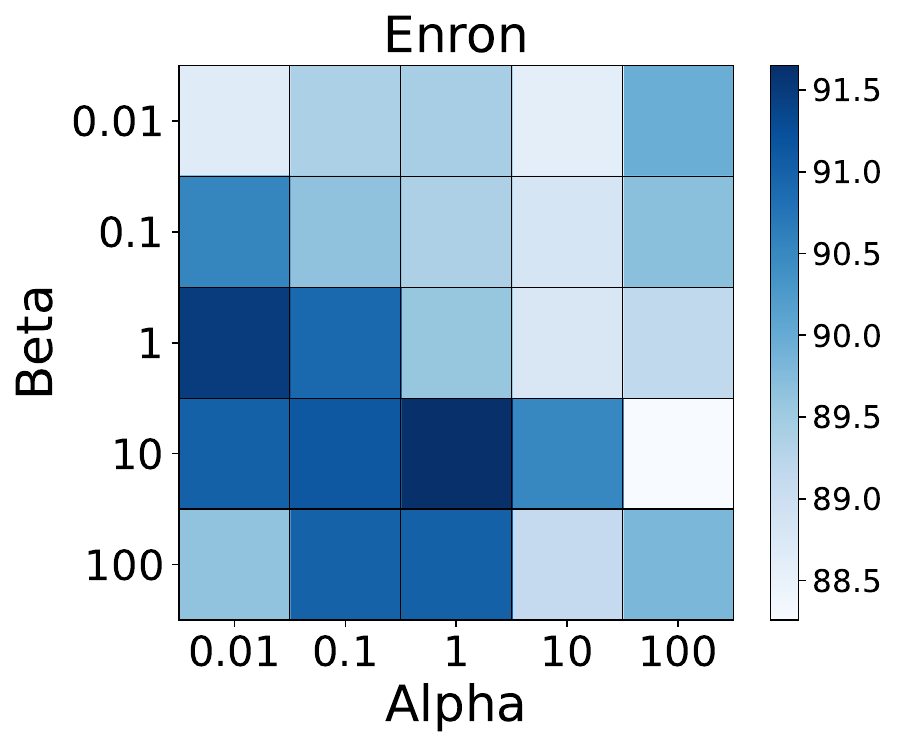}\label{fig:enron_alpha_beta}}
  \end{minipage}

  \vspace{0.5em}

  \begin{minipage}[b]{0.45\linewidth}
    \centering
    \subfigure[]{\includegraphics[width=\linewidth]{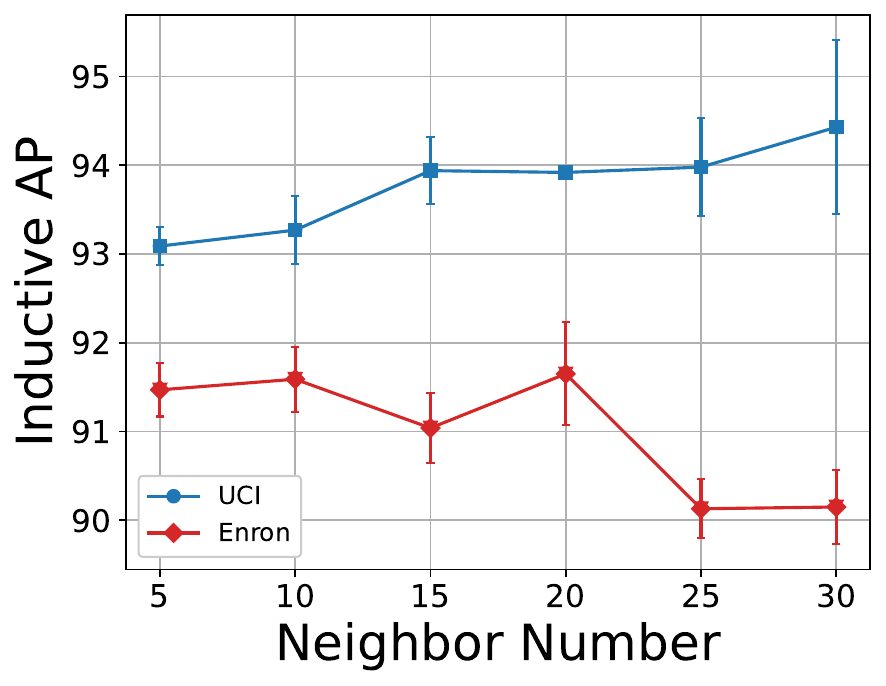}\label{fig:neigh_num}}
  \end{minipage}
  \begin{minipage}[b]{0.45\linewidth}
    \centering
    \subfigure[]{\includegraphics[width=\linewidth]{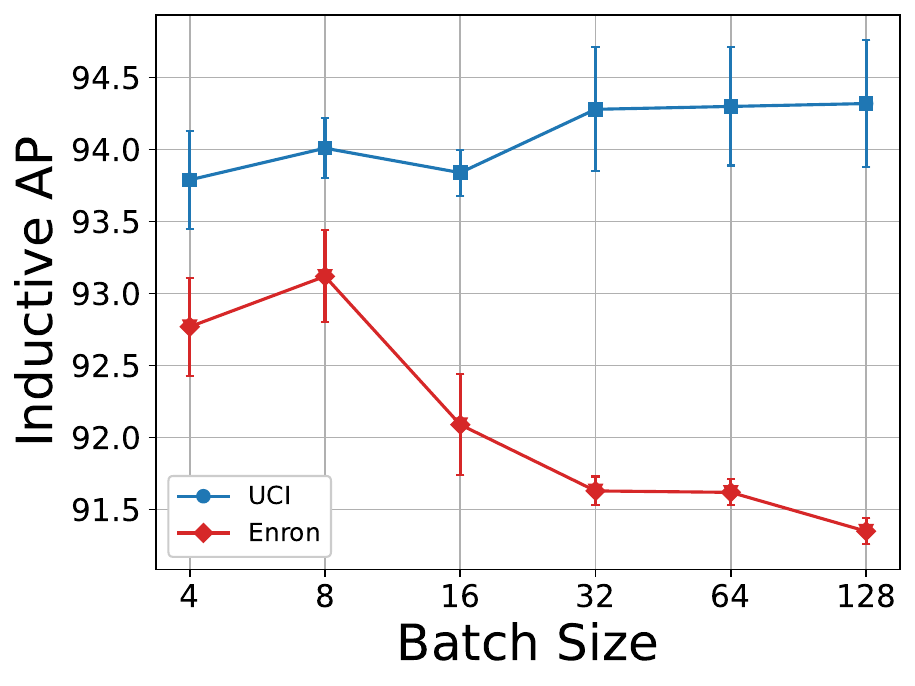}\label{fig:batch size}}
  \end{minipage}

  \caption{Parameters sensitivity analysis. The inductive AP values are described in all figures. The performance of the model changes with (a) different $\alpha$ and $\beta$ on UCI dataset, (b) different $\alpha$ and $\beta$ on Enron dataset, (c) different 1-hop neighbor number, (d) different batch size.   }
  \label{fig:sensitivity}
\end{figure}

In this subsection, we evaluate the sensitivity of the proposed \modell with respect to some key hyperparameters. The results are illustrated in Figure \ref{fig:sensitivity}, and the observations are as follows: 
(1) Equation (\ref{equation::TD2}) employs coefficients 
$\alpha$ and $\beta$ to balance recentness and intensity. We evaluate the influence of $\alpha$ and $\beta$ on UCI and Wikipedia. As shown in Figure \ref{fig:sensitivity}(a), for UCI dataset, varying settings of $\alpha$ and $\beta$ minimally influence the performance of \model. On the other hand, for Enron dataset (Figure \ref{fig:sensitivity}(b)), the smaller $\alpha$ ($\leq 1.0$) and larger $\beta$ ($\geq 1.0$) enhance the performance of \model. 
(2) We also investigate the influence of sampled 1-hop contextual node numbers on UCI and Enron datasets, as depicted in Figure \ref{fig:sensitivity}(c). On the UCI dataset, the model's performance improves with an increasing number of sampled contextual neighbors. In contrast, on the Enron dataset, the model achieves optimal performance with 20 sampled neighbors, possibly due to Enron's smaller node number.
(3) Furthermore, we explore the impact of different batch sizes on UCI and Enron datasets, as shown in Figure \ref{fig:sensitivity}(d). The model's performance remains stable on UCI. However, on the Enron dataset, performance generally deteriorates with increasing sampled neighbor numbers. This discrepancy may stem from Enron's higher average interaction intensity, suggesting that more frequent updates on recent interactions and counts could benefit proximity evaluation.

\end{document}